\documentclass[final,12pt]{clear2024} 

\title[Low-Rank Approximation of Structural Redundancy for Self-Supervised Learning]{Low-Rank Approximation of Structural Redundancy for Self-Supervised Learning}
\usepackage{times}
\usepackage{caption}
\usepackage{amsfonts,graphicx}
\usepackage{comment}
\usepackage{cases}
\usepackage{psfrag,epsf}
\usepackage{enumerate}
\usepackage{wrapfig}
\usepackage{mathtools}
\usepackage{bm}
\usepackage{physics}

\usepackage{lipsum}
\usepackage{color, colortbl}
\usepackage{natbib}
\usepackage{algorithm}
\usepackage{algorithmic}

\newtheorem{assumption}{Assumption}

\clearauthor{%
 \Name{Kang Du} \Email{kang.du@utah.edu}\\
 \addr University of Utah
 \AND
 \Name{Yu Xiang} \Email{yu.xiang@utah.edu}\\
 \addr University of Utah%
}

\usepackage{xspace}
\usepackage{bbm}
\usepackage{mathrsfs}

\def\Cov{\mathop{\rm Cov}\nolimits}%
\newcommand{\Iv}{{\bf I}}

\newcommand\independent{\protect\mathpalette{\protect\independenT}{\perp}}
\def\independenT#1#2{\mathrel{\rlap{$#1#2$}\mkern2mu{#1#2}}}

\newcommand{\bs}{\boldsymbol}

\def\l{\lambda}

\DeclareMathOperator\E{\sf E}

\newcommand{\U}{\mathrm{Unif}}

\def\textiid{i.i.d.\@\xspace}
\newcommand\iid{\ifmmode\text{ i.i.d. } \else \textiid \fi}

\newcommand{\Real}{\mathbb{R}}

\begin{document}
\maketitle

\begin{abstract}%
We study the data-generating mechanism for reconstructive SSL to shed light on its effectiveness. With an infinite amount of labeled samples, we provide a sufficient and necessary condition for perfect linear approximation. The condition reveals a full-rank component that preserves the label classes of $Y$, along with a redundant component. Motivated by the condition, we propose to approximate the redundant component by a low-rank factorization and measure the approximation quality by introducing a new quantity $\varepsilon_{s}$, parameterized by the rank of factorization $s$. We incorporate $\varepsilon_{s}$ into the excess risk analysis under both linear regression and ridge regression settings, where the latter regularization approach is to handle scenarios when the dimension of the learned features is much larger than the number of labeled samples $n$ for downstream tasks. We design three stylized experiments to compare SSL with supervised learning under different settings to support our theoretical findings. 
\end{abstract}


\section{Introduction}

Reconstructive self-supervised learning (SSL) has been highly successful in various fields~\citep{pathak2016context,vincent2010stacked,radford2018improving,devlin2018bert}, where the theme is to extract representations from unlabeled data that are potentially useful for downstream tasks. One of the major advantages of SSL is its significantly reduced dependency on labeled data. Despite abundant empirical evidence, the theoretical understanding of the performance of SSL under limited labeled data is still insufficient.

In reconstructive SSL, a \emph{pretext task} is designed to predict a target $X_{2}$ with input features $X_{1}$, which yields the learned representation $\psi(X_{1})$. Then, the downstream task is to predict the target $Y$ using  $\psi(X_{1})$. Whether the learned representation is useful for the downstream task relies on the connections between the pretext and downstream tasks. To bridge the pretext and downstream tasks, the conditional independence (CI) assumption, namely $X_{1} \independent X_{2} \,|\, Y$, has been studied in the seminal work~\citep{lee2021predicting}. For the classification setting, they show that CI is a \emph{sufficient condition} for a linear predictor to be optimal for the downstream task, that is, \emph{$\psi(X_1)$ can linearly predict $Y$ perfectly with an infinite number of samples available for the task}. Motivated by this key observation, they provide theoretical guarantees showing the superior sample complexity of SSL under general \emph{approximate} conditional independence settings.
However, a fundamental theoretical question for understanding reconstructive SSL still remains:
\begin{center}
    \emph{What is the sufficient and necessary condition on $(X_{1},X_{2},Y)$, in the classification setting, for a linear predictor to be optimal for the downstream task?\\}
\end{center}

 To address this question, it is helpful to express $X_2$ as $X_{2} = h(X_{1},Y)+ N$, where $(X_{1}, Y)$ is for an arbitrary supervised learning task and $N:= X_{2} - \E[X_{2}|X_{1},Y]$. With this expression, roughly speaking, the target $Y$ can be decoded from $X_{2}$ if $h$ is invertible in some sense. We formalize this notion of invertibility by focusing on classification problems. Our formulation allows for general dependency between $X_{2}$ and $X_{1}$ when conditioning on $Y$. Thus there are features in the learned representation $\psi(X_{1})$ that are redundant for the prediction of $Y$. For instance, for image classification problems, the redundant features may come from the background in the image; if the object of interest is surrounded by other objects in the background, a pretext task of predicting blocked patches of the image may mistakenly extract too many features from the background~\citep{pathak2016context}.  Without any constraints, a large percentage of redundant features can potentially make SSL fail. To this end, we introduce a quantity $\varepsilon_{s}$, indexed by rank $s$, for a low-rank approximation of the redundancy in the learned representation. We show how $\varepsilon_{s}$ affects the performance of SSL through both of our theoretical analysis and experiments. 

Our main contributions are summarized below. 
    \vspace{-0.6em}
\begin{enumerate}
    \item Under the classification setting, we characterize in Section~\ref{sec:iff_match} a sufficient and necessary condition for a linear predictor to be optimal in the downstream task. 
    \vspace{-0.7em}
    \item In Section~\ref{sec:low_rank}, we introduce a low-rank approximation quantity to characterize the redundancy in the learned representation.
    \vspace{-0.7em}
    \item Based on the low-rank approximation quantity, we derive finite sample bounds on the excess risk and the corresponding sample complexity for both ordinary least squares and ridge regression estimators in Section~\ref{sec:finite_sample}. 
    \vspace{-0.7em}
    \item In Section~\ref{exp:syn}, we design a simulation setting to demonstrate the effectiveness of the low-rank approximation. Our sufficient and necessary condition is partially verified through two computer vision tasks in Section~\ref{sec:exp_cv}. 
\end{enumerate}

\subsection{Related work}
Reconstructive SSL is focused on recovering deliberately concealed information in the data. In computer vision, examples include the prediction of blocked patches~\citep{pathak2016context}, recovering the color~\citep{zhang2016colorful}, denoising~\citep{vincent2010stacked}, and identifying the rotated angle~\citep{gidaris2018unsupervised}, while the simple scheme of next word prediction is widely adopted in NLP~\citep{radford2018improving,devlin2018bert}. From the theoretical perspective,~\citep{saunshi2020mathematical} and~\citep{wei2021pretrained} study how pre-trained language models yield useful representation for downstream tasks. For computer vision tasks,~\citep{pathak2016context} provides a theoretical understanding of features learned by auto-encoders under a multi-view data assumption. Under a general formulation of reconstructive SSL,~\citep{lee2021predicting} shows that CI is sufficient for a linear predictor to be optimal in the downstream task and provides finite sample analysis. Since CI often fails to hold in practical settings,~\citep{teng2022can} proposes to modify the unlabeled data to make CI hold. Their theoretical analysis suggests that the modification is hurtful rather than helpful for the performance of SSL. The other popular type of SSL is called contrastive SSL, where the goal is to learn representations that make different views of the same data point closer. The  CI assumption has been adopted in~\citep{arora2019theoretical,tosh2021contrastive} to provide theoretical guarantees for contrastive learning. In the context of contrastive SSL, CI is a natural assumption since, ideally,  two views are expected to share less information given the label. The literature on SSL is vast and we refer the readers to~\citep{gui2023survey,ozbulak2023know} for detailed reviews.


\subsection{Notation}
Throughout the paper, $\lVert\cdot\rVert$ denotes the $l_{2}$ norm for vectors or Frobenius norm for matrices. We use $\bs{0}$ and $\bs{1}$ to denote vectors or matrices of zeros and ones, respectively. For a full (column) rank matrix  $A \in \mathbb{R}^{m \times n}$ with $n < m$, we use $A^{\dagger}$ to denote its (left) pseudoinverse. Let $\Cov(X)$ denote the covariance matrix of a random vector $X$ and $\Cov(X_{1},X_{2})$ denote that of two random vectors $X_{1}$ and $X_{2}$. We use $\widetilde{\mathcal{O}}$ to hide log factors and $\lesssim$ to hide constants in inequalities. We use $\Iv_{d}$ to denote the identity matrix of size $d \times d$.   A random vector $X \in \mathbb{R}^{d}$ is said to be $\sigma^{2}$-sub-Gaussian if $\E[X]=\bs{0}$ and $\E[e^{t u^{\top}X}] \leq e^{\frac{t^2\sigma^{2}}{2}}$ for any $t \in \mathbb{R}$ and $u \in \mathbb{R}^{d}$ such that $\lVert u \rVert =1$.

\section{Problem Formulation: Reconstructive SSL}

Consider $(X_{1},X_{2},Y) \in \mathcal{X}_{1}  \times \mathcal{X}_{2} \times \mathcal{Y}$ where $X_{2}\in \mathbb{R}^{d_{2}}$ and $Y$ are the target variables for the pretext and downstream tasks, respectively, and $X_{1} \in \mathbb{R}^{d_{1}}$ is a vector of features shared by the two prediction tasks. We focus on the classification setting for the downstream task, i.e., $Y$ is categorical. For regression problems, one can consider the continuous target variable being discretized to a set of values. We use $\bar{Y} \in \bar{\mathcal{Y}} = \{y_{1},\ldots,y_{p+1}\}$ to denote the original label variable and $Y = (\mathbbm{1}_{\bar{Y}=y_{1}},\ldots,\mathbbm{1}_{\bar{Y}=y_{p}})^{\top}$ to denote its one-hot encoding with one class excluded to avoid multicolinearity as $\sum_{i=1}^{p+1}\mathbbm{1}_{\bar{Y}=y_{i}}=1$, and we will simply refer to $Y$ as the one-hot encoding of $\bar{Y}$ throughout this work. We assume $p < d_{2}$ throughout the work. For simplicity, we assume that the optimal predictors for different classes of $Y$ are not linearly dependent, i.e., $\Cov(\E[Y|X_{1}])$ has full rank; otherwise, certain classes of $Y$ can be hidden to make it hold.

Concretely, we consider the following reconstructive SSL procedure. 
\begin{enumerate}
    \item \emph{Pretext task}: Given unlabeled data, predict $X_{2}$ using $X_{1}$ under some function class $\Psi$, i.e., estimate $\psi^{*} :=  \arg\min_{\psi \in \Psi}\E[\lVert X_{2}-\psi(X_{1})\rVert^{2}]$.
    \item \emph{Downstream task}: Given $n$ labeled data, regress $Y$ on the learned representation $\psi^{*}(X_{1})$ using \emph{simple} regression functions such as linear or ridge regression.
    
\end{enumerate}

Since there is often a large amount of unlabeled data and one can adopt deep neural networks to achieve universal approximation, we fix $\psi^{*}(x):= \E[X_{2}|X_{1}=x]$ and focus on analyzing the downstream task. Due to the nature of the small (labeled) sample size of SSL, the function class for the downstream task is often assumed to have lower complexity compared to $\Psi$ (e.g., smaller parameter space).  For theoretical analysis, we consider the class of all linear functions for the downstream task similarly as in~\citep{lee2021predicting}. In practice, the advantage of SSL over supervised learning (SL) is more significant when the labeled sample size $n$ is relatively small, in which case the dimension of $\psi^{*}$ can be larger than $n$. To avoid the downstream task being ill-posed, we adopt the ridge estimator. To measure the gap between the SSL prediction and the optimal predictor $\E[Y|X_{1}]$ in infinite and finite samples, respectively, we define the approximation error and excess risk.


\begin{definition}\label{def:match}
Define the approximation error of  SSL as $ \mathsf{error}_{\text{apx}}^{*}:=  \min_{\beta} \mathsf{error}_{\text{apx}}(\beta)$, where  
\begin{equation}
   \mathsf{error}_{\text{apx}}(\beta
  ) :=    \E\left[ \left\lVert \E[Y|X_{1}] - \beta \psi^{*}(X_1)\right\rVert^2 \right] \label{def_sept}
\end{equation} 
with $\psi^{*}(x)= \E[X_{2}|X_{1}=x]$, the optimal predictor of $X_2$ given $X_1$.

\end{definition}
Note that $\psi^{*}(x)= \E[X_{2}|X_{1}=x]$ can be ensured by a function class with universal approximation power such as deep neural networks. 
\begin{definition}\label{def:exact}
We say there is an  exact matching between $Y$ and $X_{2}$ given $X_{1}$ if $\mathsf{error}_{\text{apx}}^{*}=0$.
\end{definition}

For simplicity, we will omit the intercept $b(\beta):=\E[Y]-\beta\E[X_{2}]$ throughout the work. The performance of the downstream task is usually quantified through the so-called \emph{excess risk} defined with respect to the finite sample analysis. Denote $X:=\psi^{*}(X_{1})=\E[X_{2}|X_{1}]$. Let $\bs{X}_{1} \in \mathbb{R}^{n\times d_{1}}$ and $\bs{Y} \in \mathcal{Y}^{n\times p}$ be the labeled data, and $\bs{X}: = \psi^{*}(\bs{X}_{1})\in \Real^{n\times d_2}$ denote the learned representation from pretraining. For the downstream task and $\lambda \geq 0$, let
\begin{align*}
    \hat{\beta}_{\lambda} := \arg\min_\beta \frac{1}{n}\lVert\bs{Y}-\bs{X}\beta^{\top}\rVert^2+\lambda\lVert\beta\rVert^2  = \bs{Y}^{\top} \bs{X}(\bs{X}^{\top}\bs{X} +\lambda n \Iv_{d_{2}})^{-1}.
\end{align*}

\begin{definition}\label{def:match}
 The excess risk induced by the estimator $\hat{\beta}_{\lambda}$ is defined as $\mathcal{R}(\hat{\beta}_{\lambda}): = \mathsf{error}_{\text{apx}}(\hat{\beta}_{\lambda})$. 

\end{definition}

  The term ``matching'' can be viewed in the following sense: (1) the form of nonlinearity in $\E[Y|X_{1}] $ should be captured by $\E[X_{2}|X_{1}]$; (2) the ``redundant'' nonlinearity in $\E[X_{2}|X_{1}]$ should be linearly dependent so that they can be removed through a linear transform. As a toy example, consider $\E[Y|X_{1}] = X_{1}^{2}$  and $\E[X_{2}|X_{1}] = (-X_{1}^{2}+\sin(X_{1}),\,0.5\sin(X_{1}))^{\top}$ and note they share the same quadratic term $X_{1}^{2}$, while the sine functions in $\E[X_{2}|X_{1}]$ are redundant for predicting $Y$. Observe that SSL with $\beta = (-1,\, 2)^{\top}$ extracts the quadratic term while eliminating the sine functions. In contrast, $\E[X_{2}|X_{1}] =(X_{1},\, 0.5\cos(X_{1}))^{\top}$ will not lead to an exact matching.

\begin{remark}
    This notion of predicting a subset of $X$ can be helpful for predicting $Y$ is not limited to reconstructive SSL. For instance, in a series of recent papers~\cite{du2022invariant,du2023generalized,du2023learning}, the authors have explored a similar direction from an invariance perspective for multi-environment domain adaption, which has partially motivated this study.
\end{remark}


\section{Necessary and Sufficient Condition for Exact Matching}\label{sec:iff_match}

In an attempt to demystify the matching between the pretext and downstream tasks, we propose to identify the conditions on the generating mechanism of $(X_{1}, X_{2}, Y)$ that enable an exact matching. The generating mechanism of $(X_{1}, Y)$ in a supervised learning task is often complicated, and thus we make no assumptions on how $(X_{1},Y)$ is generated and focus on the interactions between $(X_{1},Y)$ and $X_{2}$. Without loss of generality, we can write $X_2$ in the following form 
\begin{equation}
    X_{2}= h(X_{1},Y) + N, \label{func}
\end{equation}
where $h(X_{1},Y): = \E[X_{2}|X_{1},Y]$ is the regression function of $X_{2}$ on $(X_{1},Y)$ and therefore the residual variable $N:= X_{2} - h(X_{1},Y)$ satisfies $\E[N|X_{1},Y]=0$. 
 The function $h$ captures how the label $Y$ and feature $X_{1}$ are encoded into $X_{2}$.



Equation~\eqref{func} can be viewed from a causal perspective through a general structural causal model (SCM)~\citep{pearl2009causality}, $ X_{2} = f(X_{1},Y, \varepsilon)$, 
where $\varepsilon$ is a vector of exogenous variables independent of $(X_{1}, Y)$. Since this general SCM suffers from identifiability issues, we focus on~\eqref{func}, observing that $h(X_{1},Y): = \E[X_{2}|X_{1},Y] = \E[f(X_{1},Y, \varepsilon)|X_{1},Y]$.  It is important to note that~\eqref{func} is valid even when there is no underlying causal graph over $(X_{1},X_{2}, Y)$.   

Recall that $Y$ is the one-hot encoding of $\bar{Y}$. Observe that an arbitrary function $h:(\mathcal{X},\mathcal{Y}) \to \mathbb{R}^{d}$ can be equivalently written as
\begin{equation}
    h(X,Y) = \sum_{j=1}^{p}h(X,e_{j}) \mathbbm{1}_{\bar{Y}=y_j} = \sum_{j=1}^{p} h(X,e_{j}) e_{j}^{\top}Y := C^{h}(X) Y, 
\end{equation}
where we use the fact that $ \mathbbm{1}_{\bar{Y}=y_j} =  e_{j}^{\top}Y$. This simple derivation implies a one-to-one correspondence between $h$ and $C^{h}$, meaning that the general function model~\eqref{func} can be expressed as
\begin{equation}
        X_{2} = h(X_{1},Y) + N = C^{h}(X_{1})Y  +N \label{nonlinear_scm},
\end{equation}
with $C^{h}: \mathcal{X}_{1} \to \mathbb{R}^{d_{2} \times p}$. The role of the latent random matrix $C^{h}(X_{1}) $ is to encode the label variable $Y$ into $X_{2}$, thus we call $C^{h}$ the \emph{encoding function}. The expression in~\eqref{nonlinear_scm} implies the identity $\E[X_{2}|X_{1}] = C^{h}(X_{1})\E[Y|X_{1}]$, which is equivalent to
\begin{equation}
    \E[X_{2}|X_{1}] = (C^{h}(X_{1})+O(X_{1}))\E[Y|X_{1}]:= \bar{C}^{h}(X_{1})\E[Y|X_{1}], 
\label{eq:equiv_ch}
\end{equation}
for any $O: \mathcal{X}_{1} \to  \mathbb{R}^{d_{2} \times p}$  such that  $O(X_{1})\E[Y|X_{1}] = \bs{0}$. In words, the rows of $O(x)$ are orthogonal to $\E[Y|X_{1}=x]$ for $\forall x \in \mathcal{X}_{1}$. We call such $O(x)$ an orthogonal term. For instance, the orthogonality holds when $\E[Y|X_{1}] = (X_{1},X_{1}^2)^{\top}$ and each row of $O(X_{1})$ is $(-X_{1},1)$. Equation~\eqref{eq:equiv_ch} defines an equivalent class of encoding functions $\mathcal{C} = \{\bar{C}^{h}\}$ that results in the same pretext representation $\psi^{*}(X_{1}) = \E[X_{2}|X_{1}]$. This shows that such orthogonal terms do not affect the analysis of SSL, and thus we use $\doteq$ to hide the orthogonal term (added to $C^{h}$) in equations throughout the paper. 


\begin{proposition}\label{def:alter}
The exact matching in Definition~\ref{def:exact} holds if and only if $\beta C^{h}(x) \doteq  \Iv_{p}$, for $\forall x \in \mathcal{X}_{1}$ and some $\beta \in \mathbb{R}^{p \times d_{2}}$. 
\end{proposition}


Therefore, in this formulation, finding an exact matching is equivalent to inverting the encoding function $C^{h}$. Proposition~\ref{def:alter} implies that the full rank of $C^{h}(x)$  for every $x \in \mathcal{X}_{1}$ is a necessary condition for the exact matching. In the following lemma, we provide \emph{a sufficient and necessary condition for the exact matching} through a full characterization of the invertibility of $C^{h}$.

\begin{lemma}[sufficient and necessary condition for exact matching]\label{lem:iff_match}There is an exact matching between $Y$ and $X_{2}$ given $X_{1}$  if and only if
    \begin{equation}
    C^{h}(x) \doteq A \begin{bmatrix}
        \Iv_{p}\\
        R(x)
    \end{bmatrix}  \quad \text{ for }\forall x \in \mathcal{X}_{1},\label{iif_cond}
    \end{equation}    
for some invertible matrix $A \in \mathbb{R}^{d_{2}\times d_{2}}$, an arbitrary matrix function $R: \mathcal{X}_{1} \to \mathbb{R}^{(d_{2}-p)\times p}$.   
\end{lemma}

The identity mapping $\Iv_{p}$ fully preserves each class of $Y$, and $R(x)$ represents the redundancy encoded into $X_{2}$. It is worth noting that redundancy refers to the features extracted from $X_{1}$ that are predictive for $X_{2}$, but redundant for the prediction of $Y$ (given the optimal predictor $\E[Y|X_{1}]$).  In our stylized MNIST experiment in Section~\ref{sec:sty_mnist}, the dash pattern in the background is redundancy since it is useful for predicting the image orientation (i.e., $X_{2}$), but it contains no information about the label. In contrast, the dot pattern is not redundant, since it is independent of both the image orientation and the label. The lemma above reveals that the label $Y$ should be encoded into $X_{2}$ through an invertible linear mixture (i.e., $A$) of the full label information and some redundancy. When $A$ is an identity matrix, the first $p$ rows of $\psi^*(X_{1}) = \E[X_{2}|X_{1}]$ capture the full label information, thus the downstream task has a sparse solution $\beta^{*} = [\Iv_{p}, \, \bs{0}]$. However, the solutions to the downstream task may not be sparse in general, and we handle this challenge in Section~\ref{sec:low_rank}. Below are two examples with explicit forms of $C^h$. 

\begin{example}\label{exm:const}An important special case of model~\eqref{nonlinear_scm} is $X_{2} = \widetilde{C}Y+N$, where $C^{h} \equiv \widetilde{C}$ is a constant function.
In this case, the necessary and sufficient condition simplifies to the condition that $\widetilde{C}$ has full rank. Observe that $\E[X_{2}|X_{1}] =\widetilde{C} \E[Y|X_{1}]$ implies  $  \mathsf{error}_{\text{apx}}(\widetilde{C}^{\dagger}) = 0$. 

\end{example}

In Appendix~\ref{app:conditional}, we show that model $X_{2} = \widetilde{C}Y+N$ is equivalent to $\E[X_{2}|X_{1},Y]= \E[X_{2}|Y]$, which we call \emph{conditional mean independence}, which is weaker than CI, i.e., $X_{1} \independent X_{2} \,|\, Y$. The setting in Example~\ref{exm:const} has been studied in~\citep{lee2021predicting} under CI. Despite the simplicity of conditional mean independence, it can be unrealistic in practical settings, since it requires that the label $Y$ is fully encoded into $X_{2}$ with no redundant information (depending on $X_{1}$) as if the pretext and downstream tasks are two equivalent prediction tasks. Even though approximate conditional independence has been studied in~\citep{lee2021predicting}, it is unclear if the approximation provides sufficient insights into explaining why and when SSL works (or fails), since conditional independence (or constant $C^{h}$) is not a necessary condition for exact matching.

\begin{example}[partially linear model] 
Define an invertible matrix $A\in \mathbb{R}^{d_{2} \times d_{2}}$ with a column partition as $A = [A_1 | A_2 | A_3]$, where $A_{1}$ has $p$ columns, $A_{2}$ has $k$ columns such that $1 \leq k \leq d_{2}-p$, and $A_{3}$ has the rest of the columns. Let $X_{2} = A_{1}Y + A_{2} a(X_{1}) + N$, where  $a: \mathcal{X}_{1} \to \mathbb{R}^{k}$ satisfies $\E[a(X_{1})]=\bs{0}$. Its encoding function is  $C^{h}(X_{1}) = A_{1}+A_{2} a(X_{1}) \bs{1}^{\top}$ as derived below. The sufficient and necessary condition is immediately satisfied with $A$ and $R(x) = [ \bs{1}a^{\top}(x) , \bs{0}]^{\top}$,
 where $R(x)$ has $d_{2}-p -k$ all zero rows.
\vspace{-.6em}
\begin{equation}
    h(X_{1},Y)  = \sum_{j=1}^{p}(A_{2}a(X_{1})+A_{1}e_{j})e_{j}^{\top} Y = \big(A_{2}a(X_{1})\sum_{j}e_{j}^{\top} + A_{1}\sum_{j}e_{j}e_{j}^{\top} \big) Y:=C^{h}(X_{1})Y .  \nonumber
    \vspace{-.6em}
\end{equation}Even though conditional independence fails to hold since $C^{h}$ is not constant, according to Lemma~\ref{lem:iff_match}, there is still an exact matching.
\end{example}

\section{Structural Redundancy}\label{sec:low_rank}
The pretext representation $\psi^{*}(X_{1})$ is typically high-dimensional, designed to capture abundant information for various downstream tasks. Given limited labeled samples ($n \ll d_{2}$ in our notation), the downstream task is a high-dimensional linear regression problem. Without any assumptions such as sparsity of the true coefficients~\citep{tibshirani1996regression, candes2007dantzig} or low effective dimension of the features~\citep{zhang2005learning,hsu2012random}, SSL may not perform well even if the exact matching holds. Since the true coefficients are not necessarily sparse as discussed in Section~\ref{sec:iff_match}, we explore how low-rank structures in the redundancy (i.e., $R(X_{1})$) naturally lead to a low effective dimension. In particular, we adopt the definition of effective dimension from~\citep{hsu2012random} in the context of ridge regression (see details below); a closely related notion is called effective degrees of freedom~\citep{efron1986biased,hastie2009elements}. Roughly speaking, the effective dimension measures the number of features that are not linearly correlated; when it is low, a small number of labeled samples can be sufficient for reliable estimation



To simplify the notation, denote $X:=\psi^{*}(X_{1})$ with  covariance matrix $\Sigma := \Cov(X)$. Let $\{\lambda_{j}\}_{j=1}^{d_{2}}$ denote the eigenvalues of $\Sigma$. 
The population ridge estimator is given by
\begin{align*}
\beta_{\lambda} = \arg\min_{\beta} \E\left[\lVert Y-X\beta \rVert^2\right]+\lambda \lVert\beta \rVert^2 = (\Sigma + \lambda \Iv_{d_{2}})^{-1}\E[XY^{\top}].
\end{align*}
Implicitly dimension reduction is performed in ridge regression when some appropriate shrinkage parameter $\lambda$ is chosen. The reduced dimension for a chosen $\lambda$ can be quantified by the \emph{effective dimension}, defined as $d_{\lambda} = \sum_{j=1}^{d}\frac{\lambda_{j}}{\lambda_{j}+\lambda}$ for $X\in \mathbb{R}^{d}$. Note that the bias of the ridge estimator increases monotonically as $\lambda$ increases. When $\Sigma$ has exactly $s$ nonzero eigenvalues, $d_{\lambda}$ is upper bounded by $s$ for any $\lambda\geq 0$. Besides this special case, a low effective dimension can be achieved under a weak penalty (i.e., small $\lambda$) when there is a large percentage of small eigenvalues. 


In the next subsection, we demonstrate that a low effective dimension is naturally attained when the redundancy $R(X_{1})$ can be approximated by a low-rank decomposition. Our finite sample analysis on the high-dimensional setting, presented in Section~\ref{sec:finite_sample}, relies on an upper bound of the low effective dimension, utilizing a measure of the low-rank approximation introduced in the following subsection (Lemma~\ref{assump:bound_d1}). When $d_{2}<n$, $d_{\lambda}$ with $\lambda=0$ offers an interpretation of our upper bound on the excess risk and sample complexity (see details in Theorem~\ref{thm:ols} and Remark~\ref{rmk:es_zero}).

\subsection{Low-rank Approximation of Redundancy}
Recall, redundancy refers to information in $X_{1}$ useful for predicting $X_{2}$ but not for the label $Y$. For instance, in computer vision tasks, consider a label $Y$ determined by the object of interest within a surrounding background. Redundancy arises when the pretext task captures background information unrelated to the label. If the background features \emph{simple} patterns, such as sky, grassland, or beach, this redundancy can be considered low-rank. Consequently, it is relatively easy to eliminate such redundancy in downstream tasks (recall the cancellation of sine functions below Definition~\ref{def:match}). In the following, we present the technical details of the low-rank approximation of redundancy. 

Denote $\widetilde{C} := \E[C^{h}(X_{1})]$ and recall that $C^{h}(X_{1})$ reduces to $\widetilde{C} $ under conditional mean independence.  Assume that the necessary and sufficient condition in Lemma~\ref{lem:iff_match} is satisfied,  
\begin{align}
   X =  C^{h}(X_{1})\E[Y|X_{1}] 
   &=\left( \widetilde{C} + A\begin{bmatrix}
        \bs{0} \,\, 
        \left(R(X_{1})-\widetilde{R}
    \right)^{\top}\end{bmatrix}^{\top}\right)\E[Y|X_{1}] \nonumber\\ 
&= (\widetilde{C}  +  A_{p+1:d_{2}} (R(X_{1})-\widetilde{R})))\E[Y|X_{1}], \nonumber
\end{align}
where $\widetilde{R}: = \E[R(X_{1})]$ and $ A_{p+1:d_{2}} \in \mathbb{R}^{d_{2} \times (d_{2}-p)}$ denotes the last $d_{2}-p$ columns of $A$. If the (centered) redundancy $R(X_{1})-\widetilde{R}$ admits a low-rank decomposition, i.e.,  $R(X_{1})-\widetilde{R} = Bg(X_{1})$ for some $B \in \mathbb{R}^{(d_{2}-p) \times s}$ and $g: \mathcal{X}_{1} \to \mathbb{R}^{s \times p}$, where $s \ll d_{2}$, we get 
\begin{equation}
     X = (\widetilde{C} + A_{p+1:d_{2}}Bg(X_{1}))\E[Y|X_{1}], \label{eq:exact_low_rank}
\end{equation}
which has at most $p+s$ linearly independent components, as $\text{rank}(\widetilde{C}) \leq p$ and $\text{rank}(A_{p+1:d_{2}}B)\leq s$. This shows that the effective dimension of $X$ is bounded by $p+s$ for any $\lambda\geq 0 $.

Since the low-rank decomposition may not hold exactly for a chosen $s$, we identify $\hat{X}$ of the form $ (\widetilde{C} + Bg(X))\E[Y|X_{1}]$ that best approximates $X$. Specifically, for any fixed $s$ s.t. $1 \leq s \leq d_{2}-p$, we consider $B\in \mathbb{R}^{d_{2} \times s},\,g: \mathcal{X}_{1} \to \mathbb{R}^{s \times p}$, and define 
\begin{align}
  \varepsilon_{s} &: = \min_{\hat{X}} \frac{1}{d_{2}} \E\left[\left\lVert X- \hat{X} \right\rVert^2 \right]  = \min_{B,g}\frac{1}{d_{2}} \E\left[\left\lVert \left(C^{h}(X_{1}) - \widetilde{C} - Bg(X_{1})\right) \E[Y|X_{1}] \right\rVert^2\right], \label{approxi_decom}
\end{align}
with minimizers $\{(B^{*}, g^{*})\}$ and we use the shorthand $ \widetilde{X}:=(\widetilde{C} + B^*g^*(X_{1}))\E[Y|X_{1}]$ for any pair of optimizer $(B^{*}, g^{*})$.  Without loss of generality, we normalize $g(X_{1})$ and assume $\E[\lVert g(X_{1}) \rVert]=1$. The low-rank approximation error is averaged so that $\varepsilon_{s}$ does not grow with $d_{2}$. A challenge for analyzing $\varepsilon_{s}$ is that the minimizers do not have closed-form expressions. When $(X_{1},X_{2},Y)$ follows a Gaussian distribution, we show~\eqref{approxi_decom} reduces to a weighted low-rank approximation problem (with no randomness) in Appendix~\ref{app:Gaussian}. However, these weighted problems do not have closed-form expressions in general~\citep{srebro2003weighted,dutta2017problem}. Therefore, we leave the further investigations of $\varepsilon_{s}$ 
for future work. For $s=0$, we simply define 

\begin{equation}
    \varepsilon_{0}: = \frac{1}{d_{2}} \E\left[\left\lVert \left(C^{h}(X_{1}) - \widetilde{C}\right) \E[Y|X_{1}] \right\rVert^2\right], \nonumber
\end{equation}
which measures how approximately conditional mean independence (i.e., $C^{h} = \widetilde{C}$) holds. There is a tradeoff between the effective dimension of $\widetilde{X}$ (which are no greater than $s$) and the approximation quality, as the approximation level $\varepsilon_{s}$ is non-increasing as $s$ increases. An important special case of the low-rank approximation is when the encoding functions are smooth. 

\begin{example}[smooth encoding function] Consider a binary classification problem with a scalar predictor $X_{1}$, i.e., $p=d_{1}=1$, assume that  the encoding function $C^{h}: \mathbb{R} \to \mathbb{R}^{d_{2}}$ is twice continuously differentiable, then its second-order Taylor expansion at $a \in \mathbb{R}$ is given by
\begin{align}
C^{h}(x) = C^{h}(a) + \begin{bmatrix}
     \dv {C^{h}}{x} \big|_{x=a} & \dv[2] {C^{h}}{x} \big|_{x=a}
\end{bmatrix} \begin{bmatrix}
    x-a & (x-a)^2
\end{bmatrix}^{\top} +  \mathcal{O}((x-a)^3) \nonumber, 
\end{align}
where we can choose $a$ so that $C^{h}(a) = \widetilde{C}$. This provides a rank-two approximation for $C^{h}(x)-\widetilde{C}$ such that $\varepsilon_{s} = \mathcal{O}((x-a)^3) $, where $s=2$.  This example can be generalized to high-order, multi-class, and multivariate cases, and we provide the details in Appendix~\ref{app:smooth_encode}. 
    
\end{example}

To understand the impact of the size of $\varepsilon_{s}$ on how approximately the matching holds (or how small the approximation error is), we derive the following upper bound via a ridge-type estimator. Unlike ridge-type estimators used in practice, the parameter $\varepsilon_{s}$ that restricts the size of the coefficients is determined by the generating mechanism of $(X_{1},X_{2}, Y)$. 

\begin{lemma}\label{lem:bound_approx_er} Consider $B^{*}$ and $g^{*}$ corresponding to $\varepsilon_{s}$, we have
\begin{align*}
    \min_{\beta}  \mathsf{error}_{\text{apx}}(\beta)  \leq 2(p+\E[\lVert N \rVert^2])\min_{\beta} \left(\lVert \Iv_{p}-\beta\widetilde{C}\rVert^2 + \lVert\beta B^{*}\rVert^2 + \varepsilon_{s}||\beta\rVert^2  \right),
\end{align*}
where the minimum of the RHS is attained at $\beta_{s} := (\widetilde{C}^{\top}\widetilde{C}+(B^{*})^{\top}B^{*}+\varepsilon_{s} I)^{-1}\widetilde{C}$. The equality holds with the RHS being zero when $\varepsilon_{s}=0$.   
\end{lemma}


\section{Finite Sample Analysis}

\label{sec:finite_sample}

 Let $\beta^{*} \in \arg\min_{\beta} \mathsf{error}_{\text{apx}}(\beta)$ be a fixed true parameter for the downstream task. Recall that the excess risk is defined as $\mathcal{R}(\hat{\beta}_{\lambda}) = \E[\lVert\E[Y|X_{1}]-\hat{\beta}_{\lambda}X)\rVert^2]$. Under conditional mean independence, observe that $X= \E[X_{2}|X_{1}] = \widetilde{C} \E[Y|X_{1}]$ is a feature vector with at most $p$ (out of $d_{2}$) features that are linearly independent. Since the number of classes $p$ is often much smaller than the dimension of the learned representation $d_{2}$, the design matrix $\bs{X}$ for the downstream task is of low rank. This enables a finite-sample bound on the excessive risk $\widetilde{\mathcal{O}}(\frac{p}{n}\sigma^{2})$ with sample complexity $\widetilde{\mathcal{O}}(p)$~\citep{lee2021predicting}, where the bound is independent of the dimension $d_{2}$. In the following, we provide a finite-sample analysis of SSL in the general setting when conditional independence can be violated, based on the low-rank approximation defined in~\eqref{approxi_decom}.
 
 First, we introduce a few technical assumptions. Let $\Sigma$, $\widetilde{\Sigma}$, and $\bar{\Sigma}$ denote the covariance matrix of $X$, $\widetilde{X}$, and $\bar{X} := \widetilde{X} -X$, respectively. 


\begin{assumption}\label{sub_gaus} We assume $N:=Y-\E[Y|X_{1}]$ is $\sigma^2$-sub-Gaussian, and the whitened feature vectors $\widetilde{\Sigma}^{-1/2}\tilde{X}$ and $\bar{\Sigma}^{-1/2}\bar{X}$ are $\rho^2$-sub-Gaussian.~\footnote{  When $\widetilde{\Sigma}$ or $\bar{\Sigma}$ is not invertible, the whitened feature vector is defined through the generalized inverse.}  
\end{assumption}

\begin{assumption}\label{assum:bound_appro}
    There exists $\tilde{b}, \bar{b}\geq 0$ s.t. the following holds almost surely,  
    \begin{itemize}
\item $\lVert \widetilde{\Sigma} ^{-1/2}\widetilde{X}(\E[Y|X_{1}]-\beta^{*}X])^{\top}\rVert \leq \tilde{b}\sqrt{p+s}$\, ;
\item $\lVert \bar{\Sigma} ^{-1/2}\bar{X}(\E[Y|X_{1}]-\beta^{*}X])^{\top}\rVert \leq \bar{b}\sqrt{d_{2}}$\, .
\end{itemize}

\end{assumption}

\begin{remark}
  A similar assumption has been made in~\citep[Assumption~3.3]{lee2021predicting}, which is motivated by~\citep[Condition~4]{hsu2012random}. 
\end{remark}

Let $\lambda_{\text{max}}(A)$ denote the largest eigenvalue of a symmetric real matrix $A$ such that $A \neq \bs{0}$, $\lambda_{\text{min}\neq 0}(A)$ denote its smallest nonzero eigenvalue, and $\{\lambda_{i}(A)\}$ denote the set of all its eigenvalues. When $\widetilde{X}$ is good approximates of $X$, we expect  $\widetilde{\Sigma} - \Sigma$ and $\bar{\Sigma}$ to be close to zero matrices. Therefore, we consider restricting the largest eigenvalues of the two matrices, respectively. A generic bound is provided in~\citep{wolkowicz1980bounds}, that is  $\lambda_{\text{max}}(A) \leq \frac{\text{tr}(A)}{d} + \sqrt{d-1}\cdot s(A)$, where $s(A):= \frac{\text{tr}(A^2)}{d}-\frac{\text{tr}^2(A)}{d^2}$ is the variance of $\{\lambda_{i}(A)\}$. The equality holds when the $d-1$ smallest eigenvalues are equal.  However, this bound can be quite loose when $d$ is large. Instead, we make the following assumption.

\begin{assumption}\label{assump:eigenvs} For some universal constants $c_{1}\geq 0$ and $c_{2}\geq 0$,
\begin{itemize}
    \item     $ \lambda_{\text{max}}(\widetilde{\Sigma} - \Sigma) \leq c_{1} \frac{1}{d_{2}}\cdot \lvert \text{tr}(\widetilde{\Sigma} - \Sigma)\rvert $\, ; 
    \item $ \lambda_{\text{max}}(\bar{\Sigma}) \leq c_{2} \frac{1}{d_{2}}\cdot\text{tr}(\bar{\Sigma}) $\, .
\end{itemize}
\end{assumption}

Both inequalities require that the average eigenvalue is comparable to the largest eigenvalue. The assumption can be unrealistic when $\widetilde{\Sigma} - \Sigma$ or $\bar{\Sigma}$ has mostly zero eigenvalues but a few large positive eigenvalues. We explain why such a case will not happen when $s \ll d_{2}$. Case I: When $\text{rank}(\Sigma) \ll d_{2}$, there exists $s = \text{rank}(\Sigma)$ such that $\widetilde{X}= X$ and $\bar{X} = \bs{0}$, in which case the inequalities are satisfied with $c_{1}=c_{2} =0$. Case II: In settings when $\text{rank}(\Sigma)$ is comparable to $d_{2}$ (i.e., $\text{rank}(\Sigma)$ is a fraction of $d_{2}$), $\widetilde{X}$ satisfies $\text{rank}(\widetilde{\Sigma}) \leq p+s \ll d_{2}$, and thus the rank of $\widetilde{\Sigma} - \Sigma$ should be greater than $d_{2}-p-s$, meaning that most of eigenvalues are nonzero. Similarly, $\bar{X} = X - \widetilde{X}$ should have at least $d_{2}-p-s$ linearly independent components, i.e.,  $\text{rank} (\bar{\Sigma}) \geq d_{2}-p-s$. Given that $\widetilde{X}$ serves as an approximation for $X$ with a lower effective dimension, we make the following technical assumption on the rank. 

\vspace{-.5em}

\begin{assumption}\label{assum:lower_rank}
    $\text{rank}(\widetilde{\Sigma}) \leq \text{rank}(\Sigma)$ and  $\text{rank}(\bar{\Sigma}) \leq \text{rank}(\Sigma)$. 
\end{assumption}
\vspace{-.5em}
Since $X$ has $\text{rank}(\Sigma)$ linearly independent components while $\widetilde{X}$ is introduced to approximate $\text{rank}(\widetilde{\Sigma})$ independent components out of them, $\bar{X} = X - \widetilde{X}$ is expected to have less independent components than $X$.

\vspace{-.5em}
\begin{theorem}\label{thm:ols}
    Under Assumptions~\ref{sub_gaus}---~\ref{assum:lower_rank}, for any $\delta \in (0,1)$, if $n \gg \rho^4(d_{2}+\log\frac{1}{\delta})$, the excess risk of the downstream task induced by $\hat{\beta}_{0}$ is upper bounded, with probability at least $1-\delta$, by
    \begin{align*}
        \mathcal{R}(\hat{\beta}_{0}) &\leq \mathsf{error}_{\text{apx}}^{*}+ \widetilde{\mathcal{O}}\left((1+\varepsilon_{s})\frac{p+s}{n} \sigma^2+\varepsilon_{s}\frac{d_{2}}{n}\sigma^2\right).
    \end{align*}     
\end{theorem}

\vspace{-1.5em}
\begin{remark}\label{rmk:es_zero}
When $\varepsilon_{s}=0$, if $n \gg \rho^{4}(p+s+\log\frac{1}{\delta})$, we have $\mathcal{R}(\hat{\beta}_{0}) \leq  \widetilde{\mathcal{O}}\left(\frac{p+s}{n} \sigma^2\right)$.


\end{remark}

The proof of Theorem~\ref{thm:ols} follows a similar idea to that of~\citep[Theorem~$3.5$]{lee2021predicting}; a subtle yet important difference is that we consider approximation errors due to the violations of the exact matching while they consider approximation errors due to choices of the function class $\Psi$. When $\varepsilon_s=0$, the dominating rate of $\mathcal{R}(\hat{\beta}_{0})$ is $\frac{p+s}{n}\sigma^2$, which shows that SSL enjoys a similar sample complexity as shown in~\citep{lee2021predicting} even when conditional independence is violated. We have demonstrated in Lemma~\ref{lem:bound_approx_er} how the approximation error $\mathsf{error}_{\text{apx}}(\beta^{*})$ depends on the approximation level $\varepsilon_{s}$. We also provide the bound with respect to $\hat{\beta}_{\l}$, stated below. The proof is largely followed from~\citep[Theorem~16]{hsu2012random} and we only outline the main steps in Appendix~\ref{app:coro9}. 

\begin{corollary}[Informal]\label{coro:beta_l}
Under suitable assumptions, the excess risk of the downstream task induced by $\hat{\beta}_{\l}$ can be upper bounded by
    \begin{equation}
  \mathcal{R}(\hat{\beta}_{\lambda}) \leq \mathsf{error}_{\text{apx}}^{*} +\E[\lVert(\beta_{\lambda}-\beta^{*})X\rVert^2] + \mathcal{O}\left(\frac{p+s}{n}\left(1+\frac{\varepsilon_{s}}{\lambda}\right)\tilde{\sigma}^2\right), \nonumber    
\end{equation}
with high probability, where $\tilde{\sigma}^2 := \mathsf{error}_{\text{apx}}(\beta_{\lambda}) +\E[\lVert(\beta_{\lambda}-\beta^{*})X\rVert^2]+\sigma^2$. 
\end{corollary}
\vspace{-.5em}

For simplicity, the parameters that depend on the choice of $\lambda$ are omitted. The bound requires $p+s \ll n$ even though $n < d_{2}$, thus an approximation~\eqref{approxi_decom} with lower rank is expected in this more challenging setting. The term $\E[\lVert(\beta_{\lambda}-\beta^{*})X\rVert^2]$ relies on the difference between $\beta_{\lambda}$ and $\beta^{*}$, as well as the choice of $\lambda$. When $\beta_{\lambda} =\beta^{*}$ and $\varepsilon_{s}=0$, the dominating rate 
$\frac{p+s}{n}\sigma^2$ is the same as that in Remark~\ref{rmk:es_zero}. This shows that low-rank structures enable SSL to share a similar excess risk upper bound and sample complexity in low- and high-dimensional settings.

\section{Experiments}


 We propose a synthetic dataset and two computer vision tasks to examine the importance of the full rank condition on $C^{h}(x)$ and the low-rank approximation quality. Recall that a necessary condition for the exact matching is that $C^{h}(x)$ has full rank for every $x \in \mathcal{X}_{1}$. For the synthetic dataset, we ensure that $C^{h}(x)$ is of full rank and focus on the low-rank approximation. For image data, since $C^{h}$ is a latent matrix function, it is not straightforward to test whether $C^{h}(x)$ has full rank in general. To this end, we design images of simple geometric shapes, that can be seen as abstractions of real images and show how some geometric properties make the rank condition hold or fail. To further understand the importance of the low-rank approximation, we add background patterns to the MNIST dataset and show that certain patterns can lead to poor low-rank approximation. See more details of the experiments in Appendixes~\ref{sec_app:syn},~\ref{sec_app:cv_shapes}, and~\ref{sec_app:mnist}. SSL approaches have achieved superior performance on large benchmark datasets, while the function class for downstream tasks is often much larger than linear models (e.g., MLPs), which is beyond the scope of our theoretical analysis. The implementation of our experiments is provided at~\url{https://github.com/dukang4655/reconstructive_ssl}. 

\vspace{-1.5em}
\subsection{Synthetic Data}\label{exp:syn}

\begin{wrapfigure}{r}{0.4\textwidth}
\includegraphics[width=.95 \linewidth]{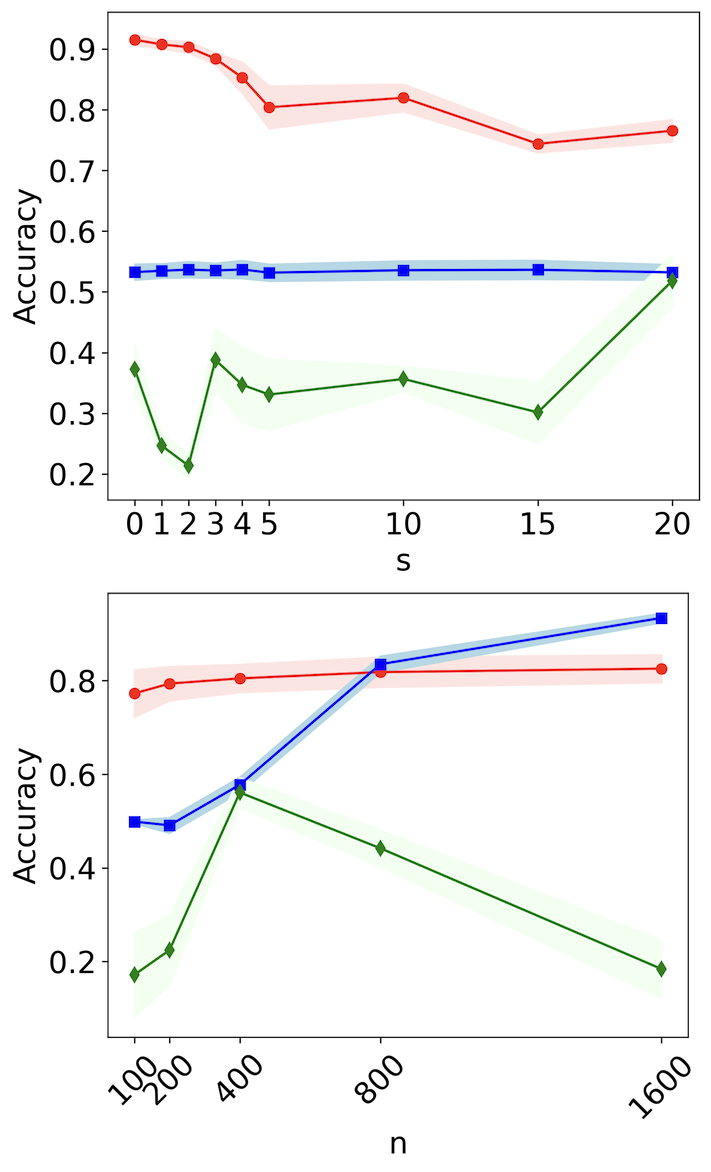}
\vspace{-1em}
  \caption{\small Setting I (top row): $n=300$ and vary $s$. SSL (red), $\text{SL}_1$ (blue), and $\text{SL}_2$ (green). $50$ repeated experiments. Solid lines: mean; shaded region: standard derivation. Setting II (bottom row): $s=5$ and $n$ varies.\label{fig:synthetic}
  } 
  \vspace{-1em}
\end{wrapfigure}
We use a synthetic dataset to verify our theoretical results when $n > d_{2}$. First, we generate $(X_{1},X_{2},Y)$ with $d_{1}=10$, $d_{2}=20$, and $p=2$, where $X_{2} = (A + B g(X_{1}))Y+N$, where $B \in \mathbb{R}^{d_{2} \times s}$. See details of the model parameters in Appendix~\ref{sec_app:syn}. We compare SSL with two supervised learning (SL) procedures in two settings: I. Fix $s =5$ and vary the labeled sample size $n \in \{100,200,400,800,1600\}$; II. Fix $n=300$ and vary the low-rank approximation by considering $s\in \{1,2,\ldots,5,10,15,20\}$. We consider two supervised learning procedures. $\text{SL}_2$: Predicting $Y$ by $X_{1}$, and $\text{SL}_2$: Predicting $Y$ by $(X_{1},X_{2})$. We use MLPs for the pretext task and the two supervised learning procedures. In Fig.~\ref{fig:synthetic}, the performance of $\text{SL}_1$ is roughly invariant with respect to $s$ since it does not use $X_{2}$ for the prediction, making it more robust than $\text{SL}_2$. This verifies that predictions using the parents of $Y$ as predictors (which we call causal predictions) are more robust than non-causal predictions under a small sample size. The superior performance of SSL degrades as $s$ increases. When $s<d_{2}-p = 18$, we have $\varepsilon_{s}=0$ according to the factorization in~\eqref{eq:exact_low_rank}. In Fig.~\ref{fig:synthetic}, when $s=20$, a low-rank approximation~\eqref{approxi_decom} could lead to a large approximation error $\varepsilon_{\bar{s}}$ for $\bar{s}\leq 18$. As a consequence, the advantage of SSL over SL$_{1}$ is much smaller compared with the case with $s=0$. This indicates that a good low-rank approximation is not only sufficient but also necessary for SSL. In the other setting when $s$ is fixed to $5$ but the sample size $n$ varies, as shown in Fig.\ref{fig:synthetic}, the performance of SSL improves slowly with the increasing sample size, since it already achieves performance that to close to the optimal (i.e., the performance of SL$_{1}$ with a large $n$) under a small $n$. SL$_{1}$ starts to catch up with SSL when $n\geq 800$, while the accuracy of SL$_{2}$ does not consistently improve as $n$ increases.


\subsection{Computer Vision Tasks}\label{sec:exp_cv}

\subsubsection{Geometric Shapes (On the Rank Condition)}

    In computer vision applications, it is common that the dimension of the learned representation is much larger than the labeled sample size (i.e., $d_{2} \gg n$) for SSL. We design a simple task to help understand how the patterns in an image make SSL work or fail; the task is inspired by~\citep{gidaris2018unsupervised}, where the pretext task is to predict the rotated angles of images. We consider $X_{1}$ as a random image of some objects, where the objects have random sizes and random locations. The goal is to classify the shape of the object $\bar{Y}$. We created $X_{2}$ by randomly rotating $X_{1}$ by $0$ or $90$ degrees. Observe that the location and the size of an object are redundant features for predicting its shape and orientation. This stylized setting, even though much simpler than real-world images, is designed this way on purpose in order for the redundancy variable to have low-rank approximations. According to~\eqref{nonlinear_scm}, the $j^{th}$ column of $C^{h}(x)$ can be viewed as a feature vector for the rotation angle corresponding to the $j^{th}$ class of objects. Since we consider the classification of two classes (i.e., $p=2$), the condition requires that the two feature vectors should not be similar. To verify this necessary condition, we consider two pairs of objects:

 \begin{figure}[h]
\centering\includegraphics[width=0.99 \linewidth]{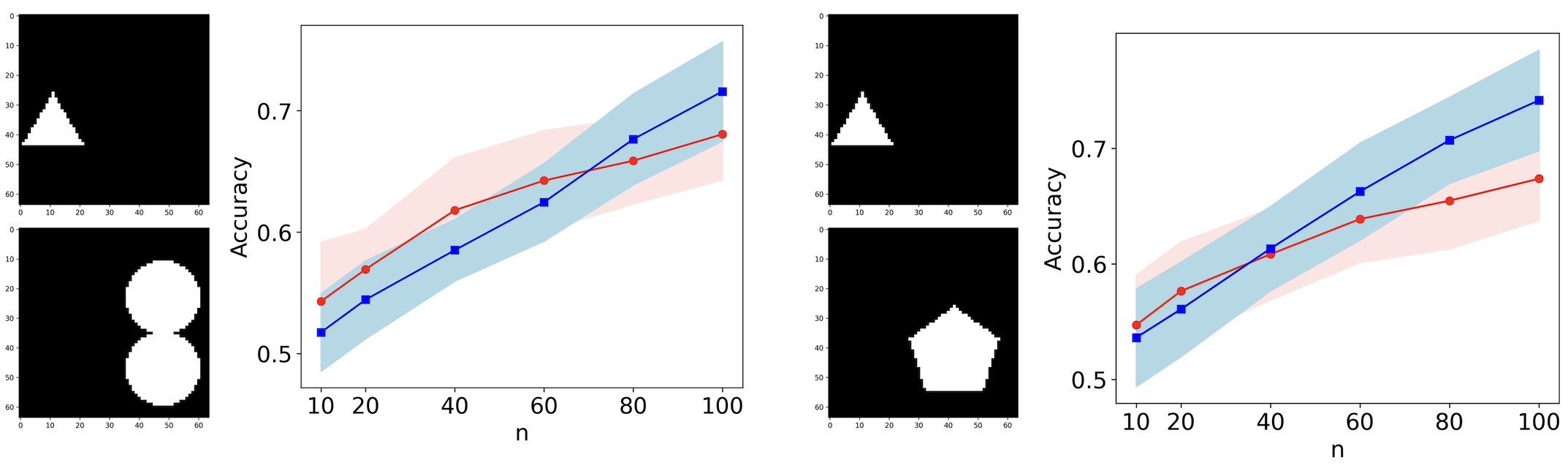}
    \vspace{-.5em}
  \caption{\small (a) Left: Triangle vs. Tangent Circles. (b) Right: Triangle vs. Pentagon. $50$ repeated experiments. SSL (red) and $\text{SL}$ (blue).\label{fig:cv_vary_n_circle}
  } 
    \vspace{-1em}
  \end{figure}

    \smallskip
   \noindent{\bf Triangle vs. Tangent Circles (Fig.~\ref{fig:cv_vary_n_circle}(a)).} In this case, the identification of orientation is based on characteristics specific to the two shapes. For triangles, it is natural to focus on the edges and vertices, while those characteristics are not even defined for circles. Thus, we think the rank condition is approximately satisfied. We examine this observation using Grad-CAM~\citep{selvaraju2017grad} that visualizes the contributing features that the model extracted from the image (see Fig.~\ref{fig:heat_shapes} from Appendix~\ref{sec_app:cv_shapes}). Similarly for the pair of objects below.

       \smallskip 
  \noindent{\bf Triangle vs. Pentagon (Fig.~\ref{fig:cv_vary_n_circle}(b))} In this case, the orientation of the two objects can be identified in similar ways, mainly based on the edges and vertices. In this case, the columns of $C^{h}(x)$ are close to linearly dependent for different $x \in \mathcal{X}_{1}$. As a consequence, the necessary condition for the exact matching is violated.

  We compare SSL with SL under different labeled sample sizes $ n \in \{10,20,40,60,80, 100\}$. For Triangle vs. Tangent Circles,  as shown in Fig.~\ref{fig:cv_vary_n_circle}, SSL consistently outperforms SL for small sample sizes (i.e., $n \leq 60$). The performance of SSL improves slower compared with SL for sufficiently large $n$, since the prediction error of SSL will be dominated by the population error instead of the estimation error.  Recall that our finite-sample bound on the excess risk in Corollary~\ref{coro:beta_l} converges to the sum of the approximation error and an error term depending on $\lambda$ as $n \to \infty$. In contrast, SSL behaves very differently for Triangle vs. Pentagon.  From Fig.~\ref{fig:cv_vary_n_circle}(b), the accuracy of SSL improves slower as $n$ increases, potentially due to the large approximation error. Overall, SSL has almost no advantage over SL. This experiment shows that the rank condition is crucial for SSL.


\subsubsection{Stylized MNIST (On the Low-Rank Approximation)} \label{sec:sty_mnist}

\begin{wrapfigure}{r}{0.43\textwidth}
\includegraphics[width=.99 \linewidth]{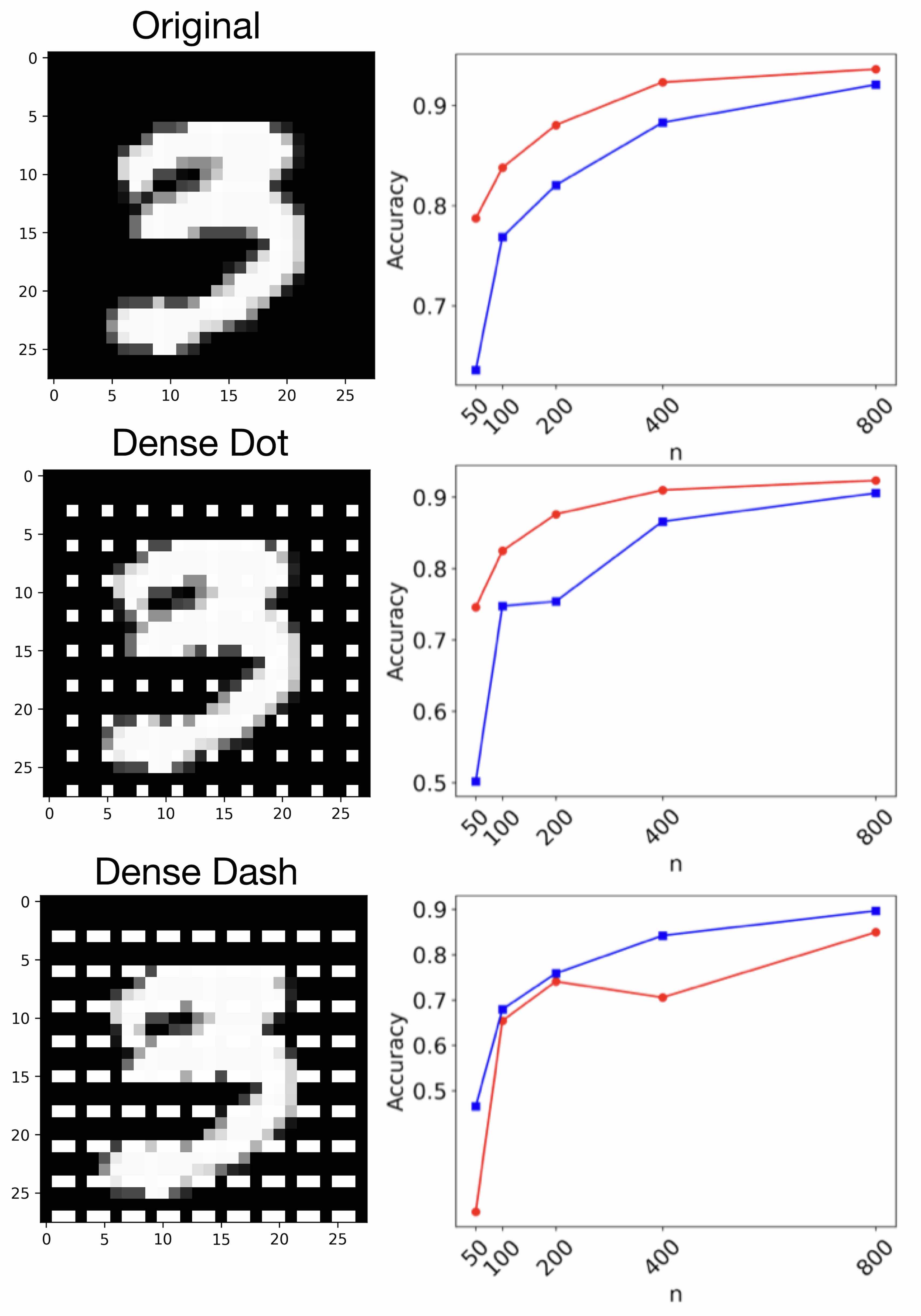}
\vspace{-1.8em}
  \caption{\small SSL(red) and SL(blue). \label{fig:minist_example}
  } 
  \vspace{-1em}
\end{wrapfigure} To examine how the redundancy affects the performance of SSL, we consider the same rotation prediction task for a stylized MNIST dataset illustrated in Fig.~\ref{fig:minist_example}, where the density of the background pattern varies randomly. A key observation is that the dot pattern does not help identify the image orientation, so it is not encoded into the orientation variable as redundancy. On the contrary, one can tell the orientation of the image simply by the orientation of the dash pattern, thus the pretext task will extract features from the dash pattern as redundant information. Again, we use Grad-CAM to visualize our observation in Fig.~\ref{fig:mnist_heat} from Appendix~\ref{sec_app:mnist}. Consequently, a dense dash pattern can lead to poor low-rank approximation. In Fig.~\ref{fig:mnist_resutls} from Appendix~\ref{sec_app:mnist}, the performance of SSL is almost invariant to the density of the dot pattern while the performance of SL drops as the density increases. In contrast, SSL is quite sensitive to a sparse dash pattern and the performance gets worse as the density increases (see Fig.~\ref{fig:minist_example}). We have tested the dot vs. dash patterns for the geometric shape images, and similar results are observed as shown in Fig.~\ref{fig:geo_pat_resutls} from Appendix~\ref{sec_app:cv_shapes}. 



\section{Discussion}


Many important questions remain to be studied and we list a few in this section. One natural next step is to study nonlinear function classes for the downstream task and characterize the corresponding sufficient and necessary conditions. Since our theoretical results can potentially provide guidance for developing SSL procedures, especially for designing pretext tasks, it would be worthwhile to design systematic and extensive experiments to better bridge the theories and practical designs. Besides the superior performance under limited labeled samples, the other major advantage of SSL is that the learned representation can be useful for a diverse class of downstream tasks; a theoretical understanding of its ability to generalize to new tasks or unseen environments (e.g., by exploiting invariance~\cite{du2023learning}) is of great importance. 


\bibliography{ref.bib}

\begin{thebibliography}{28}
\providecommand{\natexlab}[1]{#1}
\providecommand{\url}[1]{\texttt{#1}}
\expandafter\ifx\csname urlstyle\endcsname\relax
  \providecommand{\doi}[1]{doi: #1}\else
  \providecommand{\doi}{doi: \begingroup \urlstyle{rm}\Url}\fi

\bibitem[Arora et~al.(2019)Arora, Khandeparkar, Khodak, Plevrakis, and Saunshi]{arora2019theoretical}
Sanjeev Arora, Hrishikesh Khandeparkar, Mikhail Khodak, Orestis Plevrakis, and Nikunj Saunshi.
\newblock A theoretical analysis of contrastive unsupervised representation learning.
\newblock \emph{arXiv preprint arXiv:1902.09229}, 2019.

\bibitem[Candes and Tao(2007)]{candes2007dantzig}
Emmanuel Candes and Terence Tao.
\newblock The dantzig selector: Statistical estimation when p is much larger than n.
\newblock 2007.

\bibitem[Devlin et~al.(2018)Devlin, Chang, Lee, and Toutanova]{devlin2018bert}
Jacob Devlin, Ming-Wei Chang, Kenton Lee, and Kristina Toutanova.
\newblock {BERT}: Pre-training of deep bidirectional transformers for language understanding.
\newblock \emph{arXiv preprint arXiv:1810.04805}, 2018.

\bibitem[Du and Xiang(2022)]{du2022invariant}
Kang Du and Yu~Xiang.
\newblock An invariant matching property for distribution generalization under intervened response.
\newblock In \emph{2022 30th European Signal Processing Conference (EUSIPCO)}, pages 1387--1391. IEEE, 2022.

\bibitem[Du and Xiang(2023{\natexlab{a}})]{du2023generalized}
Kang Du and Yu~Xiang.
\newblock Generalized invariant matching property via lasso.
\newblock In \emph{ICASSP 2023-2023 IEEE International Conference on Acoustics, Speech and Signal Processing (ICASSP)}, pages 1--5. IEEE, 2023{\natexlab{a}}.

\bibitem[Du and Xiang(2023{\natexlab{b}})]{du2023learning}
Kang Du and Yu~Xiang.
\newblock Learning invariant representations under general interventions on the response.
\newblock \emph{IEEE Journal on Selected Areas in Information Theory}, 2023{\natexlab{b}}.

\bibitem[Dutta and Li(2017)]{dutta2017problem}
Aritra Dutta and Xin Li.
\newblock On a problem of weighted low-rank approximation of matrices.
\newblock \emph{SIAM Journal on Matrix Analysis and Applications}, 38\penalty0 (2):\penalty0 530--553, 2017.

\bibitem[Efron(1986)]{efron1986biased}
Bradley Efron.
\newblock How biased is the apparent error rate of a prediction rule?
\newblock \emph{Journal of the American statistical Association}, 81\penalty0 (394):\penalty0 461--470, 1986.

\bibitem[Gidaris et~al.(2018)Gidaris, Singh, and Komodakis]{gidaris2018unsupervised}
Spyros Gidaris, Praveer Singh, and Nikos Komodakis.
\newblock Unsupervised representation learning by predicting image rotations.
\newblock \emph{arXiv preprint arXiv:1803.07728}, 2018.

\bibitem[Gui et~al.(2023)Gui, Chen, Cao, Sun, Luo, and Tao]{gui2023survey}
Jie Gui, Tuo Chen, Qiong Cao, Zhenan Sun, Hao Luo, and Dacheng Tao.
\newblock A survey of self-supervised learning from multiple perspectives: Algorithms, theory, applications and future trends.
\newblock \emph{arXiv preprint arXiv:2301.05712}, 2023.

\bibitem[Hastie et~al.(2009)Hastie, Tibshirani, and Friedman]{hastie2009elements}
Trevor Hastie, Robert Tibshirani, and Jerome~H Friedman.
\newblock \emph{The elements of statistical learning: data mining, inference, and prediction}, volume~2.
\newblock Springer, 2009.

\bibitem[Hsu et~al.(2012)Hsu, Kakade, and Zhang]{hsu2012random}
Daniel Hsu, Sham~M Kakade, and Tong Zhang.
\newblock Random design analysis of ridge regression.
\newblock In \emph{Conference on learning theory}, pages 9--1. JMLR Workshop and Conference Proceedings, 2012.

\bibitem[Lee et~al.(2021)Lee, Lei, Saunshi, and Zhuo]{lee2021predicting}
Jason~D Lee, Qi~Lei, Nikunj Saunshi, and Jiacheng Zhuo.
\newblock Predicting what you already know helps: Provable self-supervised learning.
\newblock \emph{Advances in Neural Information Processing Systems}, 34:\penalty0 309--323, 2021.

\bibitem[Ozbulak et~al.(2023)Ozbulak, Lee, Boga, Anzaku, Park, Van~Messem, De~Neve, and Vankerschaver]{ozbulak2023know}
Utku Ozbulak, Hyun~Jung Lee, Beril Boga, Esla~Timothy Anzaku, Homin Park, Arnout Van~Messem, Wesley De~Neve, and Joris Vankerschaver.
\newblock Know your self-supervised learning: A survey on image-based generative and discriminative training.
\newblock \emph{arXiv preprint arXiv:2305.13689}, 2023.

\bibitem[Pathak et~al.(2016)Pathak, Krahenbuhl, Donahue, Darrell, and Efros]{pathak2016context}
Deepak Pathak, Philipp Krahenbuhl, Jeff Donahue, Trevor Darrell, and Alexei~A Efros.
\newblock Context encoders: Feature learning by inpainting.
\newblock In \emph{Proceedings of the IEEE conference on computer vision and pattern recognition}, pages 2536--2544, 2016.

\bibitem[Pearl(2009)]{pearl2009causality}
Judea Pearl.
\newblock \emph{Causality}.
\newblock Cambridge University Press, 2009.

\bibitem[Radford et~al.(2018)Radford, Narasimhan, Salimans, Sutskever, et~al.]{radford2018improving}
Alec Radford, Karthik Narasimhan, Tim Salimans, Ilya Sutskever, et~al.
\newblock Improving language understanding by generative pre-training.
\newblock 2018.

\bibitem[Saunshi et~al.(2020)Saunshi, Malladi, and Arora]{saunshi2020mathematical}
Nikunj Saunshi, Sadhika Malladi, and Sanjeev Arora.
\newblock A mathematical exploration of why language models help solve downstream tasks.
\newblock \emph{arXiv preprint arXiv:2010.03648}, 2020.

\bibitem[Selvaraju et~al.(2017)Selvaraju, Cogswell, Das, Vedantam, Parikh, and Batra]{selvaraju2017grad}
Ramprasaath~R Selvaraju, Michael Cogswell, Abhishek Das, Ramakrishna Vedantam, Devi Parikh, and Dhruv Batra.
\newblock Grad-{C}{A}{M}: Visual explanations from deep networks via gradient-based localization.
\newblock In \emph{Proceedings of the IEEE international conference on computer vision}, pages 618--626, 2017.

\bibitem[Srebro and Jaakkola(2003)]{srebro2003weighted}
Nathan Srebro and Tommi Jaakkola.
\newblock Weighted low-rank approximations.
\newblock In \emph{Proceedings of the 20th international conference on machine learning (ICML-03)}, pages 720--727, 2003.

\bibitem[Teng et~al.(2022)Teng, Huang, and He]{teng2022can}
Jiaye Teng, Weiran Huang, and Haowei He.
\newblock Can pretext-based self-supervised learning be boosted by downstream data? a theoretical analysis.
\newblock In \emph{International Conference on Artificial Intelligence and Statistics}, pages 4198--4216. PMLR, 2022.

\bibitem[Tibshirani(1996)]{tibshirani1996regression}
Robert Tibshirani.
\newblock Regression shrinkage and selection via the lasso.
\newblock \emph{Journal of the Royal Statistical Society Series B: Statistical Methodology}, 58\penalty0 (1):\penalty0 267--288, 1996.

\bibitem[Tosh et~al.(2021)Tosh, Krishnamurthy, and Hsu]{tosh2021contrastive}
Christopher Tosh, Akshay Krishnamurthy, and Daniel Hsu.
\newblock Contrastive learning, multi-view redundancy, and linear models.
\newblock In \emph{Algorithmic Learning Theory}, pages 1179--1206. PMLR, 2021.

\bibitem[Vincent et~al.(2010)Vincent, Larochelle, Lajoie, Bengio, Manzagol, and Bottou]{vincent2010stacked}
Pascal Vincent, Hugo Larochelle, Isabelle Lajoie, Yoshua Bengio, Pierre-Antoine Manzagol, and L{\'e}on Bottou.
\newblock Stacked denoising autoencoders: Learning useful representations in a deep network with a local denoising criterion.
\newblock \emph{Journal of machine learning research}, 11\penalty0 (12), 2010.

\bibitem[Wei et~al.(2021)Wei, Xie, and Ma]{wei2021pretrained}
Colin Wei, Sang~Michael Xie, and Tengyu Ma.
\newblock Why do pretrained language models help in downstream tasks? an analysis of head and prompt tuning.
\newblock \emph{Advances in Neural Information Processing Systems}, 34:\penalty0 16158--16170, 2021.

\bibitem[Wolkowicz and Styan(1980)]{wolkowicz1980bounds}
Henry Wolkowicz and George~PH Styan.
\newblock Bounds for eigenvalues using traces.
\newblock \emph{Linear algebra and its applications}, 29:\penalty0 471--506, 1980.

\bibitem[Zhang et~al.(2016)Zhang, Isola, and Efros]{zhang2016colorful}
Richard Zhang, Phillip Isola, and Alexei~A Efros.
\newblock Colorful image colorization.
\newblock In \emph{Computer Vision--ECCV 2016: 14th European Conference, Amsterdam, The Netherlands, October 11-14, 2016, Proceedings, Part III 14}, pages 649--666. Springer, 2016.

\bibitem[Zhang(2005)]{zhang2005learning}
Tong Zhang.
\newblock Learning bounds for kernel regression using effective data dimensionality.
\newblock \emph{Neural computation}, 17\penalty0 (9):\penalty0 2077--2098, 2005.

\end{thebibliography}

\appendix
\section{Conditional Mean Independence}\label{app:conditional}

Recall that when $C^{h}$ is a constant function, we have $X_{2} = \widetilde{C} Y+N$ and we require $\E[N|X_{1},Y]=0$ rather than $\E[N|Y]=0$.

\begin{proposition}\label{prop1}
Model~\eqref{nonlinear_scm} holds with 
$C^{h}$ being a constant function if and only if $\E[X_{2}|X_{1},Y]= \E[X_{2}|Y]$.   
\end{proposition}

\begin{proof}
    First, assume that $X_{2} = \widetilde{C}Y +N$ holds with $\E[N|X_{1},Y]=0$. It follows immediately that
\begin{equation}
    \E[X_{2}|X_{1},Y] = \E[\widetilde{C}Y+N|X_{1},Y] = \widetilde{C}Y \nonumber
\end{equation}
and $\E[X_{2}|Y]= \E[\widetilde{C}Y+N|Y] = \widetilde{C}Y$, where we use the fact that $\E[N|Y]=\E[\E[N|X_{1},Y]|Y] = 0$. Thus we have $\E[X_{2}|X_{1},Y]= \E[X_{2}|Y]$. Now we prove the other direction. Assume that $\E[X_{2}|X_{1},Y]= \E[X_{2}|Y]$, then
\begin{equation}
    \E[X_{2}|X_{1},Y] = \E[X_{2}|Y] = \sum_{i=1}^{p}\E[X_{2}|Y=e_i]\mathbbm{1}_{\bar{Y}=y_i} := \widetilde{C} Y,  \nonumber
\end{equation}
where $\widetilde{C}$ has columns $\E[X_{2}|Y=e_i]$'s, which implies model~\eqref{prop1} in the form of $X_2 = \widetilde{C}Y+N$.   
\end{proof}

\section{Proof of Proposition~\ref{def:alter}}

\begin{proof}
First, according to Definition~\ref{def:exact}, the exacting matching is equivalent to 
\begin{equation}
    \E[Y|X_{1}] = \beta \E[X_{2}|X_{1}] \label{eq:def2_equiv}
\end{equation}
for some $\beta \in \mathbb{R}^{p \times d_{2}}$. By our assumption that $\Cov(\E[Y|X_{1}])$ has full rank, $\beta$ has to have full rank when the exact matching holds. Plugging $\E[X_{2}|X_{1}] = C^{h}(X_{1})\E[Y|X_{1}]$ into~\eqref{eq:def2_equiv}, we have
    \begin{equation}
        (\Iv_{p}- \beta C^{h}(X_{1}))\E[Y|X_{1}] = \bs{0}, \nonumber
    \end{equation}
for some $\beta \in \mathbb{R}^{p \times d_{2}}$. Equivalently,
\begin{equation}
     \Iv_{p}- \beta C^{h}(X_{1}) = \bar{O}(X_{1}) \label{eq:pf_equiv}
\end{equation}
for some $\bar{O}: \mathcal{X}_{1} \to  \mathbb{R}^{p \times p}$  such that  $\bar{O}(X_{1})\E[Y|X_{1}] = \bs{0}$. Since $\beta$ has full rank,~\eqref{eq:pf_equiv} is equivalent to
\begin{equation}
     \beta (C^{h}(X_{1}) - \beta^{-1}\bar{O}(X_{1})) =  \Iv_{p},
\end{equation}
where $\beta^{-1}$ denotes the right inverse of $\beta$ and $O(X_{1}):= \beta^{-1}\bar{O}(X_{1})$ satisfies $ O(X_{1}) \E[Y|X_{1}] = \bs{0}$. 
\end{proof}

\section{Proof of Lemma~\ref{lem:iff_match}}


\begin{proof}
For simplicity of notation, we prove the lemma without considering the orthogonal term, while the orthogonal term can be directly added to the final expression of $C^{h}$. Recall the dimension of $\beta$ is $p$ by $d_{2}$.  According to Proposition~\ref{def:alter}, it is sufficient to show that $\beta C^{h}(x) = \Iv_{p}$ is equivalent to~\eqref{iif_cond}.  The ``if'' part is immediate since $\beta = [C^{-1},\bs{0}] A^{-1}$ will lead to $\beta C^{h}(x) = \Iv_{p}$, $\forall x \in \mathcal{X}_{1}$. In the following, we prove the other direction. When the exact matching holds, recall that the full rank of $\Cov(\E[Y|X_{1}])$ implies that $\beta$ has full row rank since $p<d_{2}$. The QR decomposition of $\beta^{\top}$ gives $\beta =[\bar{C}, \bs{0}] \bar{A}$, where $\bar{C} \in \mathbb{R}^{p \times p}$ is an invertible lower-triangular matrix and $\bar{A}\in \mathbb{R}^{d_{2} \times d_{2}}$ is an orthonormal matrix. Then, $\beta C^{h}(x) = \Iv_{p}$ implies $[\bar{C}, \bs{0}]B(x) = \Iv_{p}$ with $B(x) := \bar{A}C^{h}(x)$. Due to the zero columns in $[\bar{C}, \bs{0}]$, the first $p$ rows of $B(x)$ has to be $\bar{C}^{-1}$, while the other rows, denoted by $R(x)$, can be arbitrary. Therefore, we obtain
\begin{equation}
    C^{h}(x) = \bar{A}^{-1}B(x) =  \bar{A}^{-1}\begin{bmatrix}
        \bar{C}^{-1}\\
        R(x)
    \end{bmatrix} =  \widetilde{A}^{-1}\begin{bmatrix}
        \Iv_{p}\\
        R(x)
    \end{bmatrix},\quad \forall x\in\mathcal{X}_{1},  \nonumber
\end{equation}
where $\widetilde{A}^{-1}$ is the product of $\bar{A}^{-1}$ and some elementary matrices introduced to transform $\bar{C}^{-1}$ to an identity matrix $\Iv_{p}$. 
\end{proof}

\section{Low-rank approximation of Smooth Encoding Functions}\label{app:smooth_encode}

In this section, we study how the smoothness of the encoding function enables the low-rank approximation. Specifically, we will construct the approximation in~\eqref{approxi_decom} explicitly with polynomial functions. For simplicity, we present the idea for second-order approximation, while the higher-order cases can be derived in a similar manner. 

Let $C^{h}: \mathcal{X}_{1} \to \mathbb{R}^{d_{2} \times p}$ be a twice continuously differentiable matrix function, the Taylor expansion of its $j^{th}$ column $C^{h}_{j}$ at $a \in \mathcal{X}_{1}$ is given by
\begin{align}
  C_{j}^{h}(x) &= C_{j}^{h}(a) + (C_{j}^{h})'(a)(x-a) + \begin{bmatrix}
       \frac{1}{2}   (x-a)^{\top}(C_{1j}^{h})''(a)(x-a)\\
\cdots\\
 \frac{1}{2}   (x-a)^{\top}(C_{d_{2}j}^{h})''(a)(x-a)
\end{bmatrix}+ \mathcal{O}(||x-a||^3) \nonumber\\
  &:=C_{j}^{h}(a) + A_{j}(a)\phi(x) + \mathcal{O}(||x-a||^3) \label{eq_smooth_matrix},
\end{align}
where the $i^{th}$ row of $(C_{j}^{h})'(a) \in \mathbb{R}^{d_{2} \times d_{1}}$ is the derivative of $C^{h}_{ij}$ evaluated at $x=a$ and $(C_{ij}^{h})''(a)$ is the Hessian matrix of $C_{ij}^{h}$ evaluated at $x=a$. We represent $ C_{j}^{h}(x)$ in a matrix form in~\eqref{eq_smooth_matrix} by introducing $\phi(x) = (x_{1}-a_{1},x_{2}-a_{2},\ldots,x_{d_{1}}-a_{d_{1}},(x_{1}-a_{1})^2, (x_{1}-a_{1})(x_{2}-a_{2}),(x_{1}-a_{1})(x_{3}-a_{3}), \ldots,(x_{d_{1}-1}-a_{d_{1}-1})(x_{d_{1}}-a_{d_{1}}),(x_{d_{1}}-a_{d_{1}})^2)^{\top} \in \mathbb{R}^{d_{1}+d_{1}^2}$ and the coefficient matrix $A_{j}(a) \in \mathbb{R}^{d_{2} \times (d_{1}+d_{1}^2)}$ consisting of the (scaled) first two order derivatives.  This allows us to approximate $C^{h}$ by
\begin{equation}
       C^{h}(x) = C^{h}(a)+ \begin{bmatrix}
       A_{1}(a)\phi(x), A_{2}(a)\phi(x), \ldots, A_{p}(a)\phi(x)  \nonumber
   \end{bmatrix} + \mathcal{O}(||x-a||^3).
\end{equation}
Let the maximum rank of the matrices $\{A_{i}(a): i=1,\ldots,p\}$ be $s$, then there exists $B \in \mathbb{R}^{d_{2} \times s}$ and $\{D_{i}(a) \in \mathbb{R}^{s \times (d_{1}+d_{1}^2)}: i=1,\ldots,p\}$ such that
\begin{equation}
    C^{h}(x) =  C^{h}(a) + B  \begin{bmatrix}
       D_{1}(a)\phi(x), D_{2}(a)\phi(x), \ldots, D_{p}(a)\phi(x)  \label{eq:approx_tay}
   \end{bmatrix} + \mathcal{O}(||x-a||^3),
\end{equation}
which is enabled by the decomposition $A_{i}(a) = BD_{i}(a)$ for each $i$. With the continuity of $C^{h}(x)$, we can fix $a$ so that $C^{h}(a) = \widetilde{C}$ by the mean value theorem. The high-order reminder term can be ignored when the third derivatives of $C^{h}$ are all zeros. If this is not the case, we can include high-order terms in $\phi$ until the reminder term is small enough. If the maximum rank $s$ is not small, one can still consider the low-rank approximation of $\{A_{i}(a)\}$, and there will be an additional approximation error term in~\eqref{eq:approx_tay}.

\section{SSL under Gaussian Distribution}\label{app:Gaussian}

Even though our formulation focuses on the classification setting, the extension of our results to the Gaussian case is straightforward. Assume $\{X_{1},X_{2},Y\}$ are jointly Gaussian, where $Y$ is a scalar Gaussian variable. Let $\Sigma_{XZ} := \Cov(X,Z)$ for any random vectors $X$ and $Z$. Then, $X := \E[X_{2}|X_{1}] = \Sigma_{X_{2}X_{1}}\Sigma_{X_{1}X_{1}}^{-1}X_{1}$ and $\E[Y|X_{1}] = \Sigma_{YX_{1}}\Sigma_{X_{1}X_{1}}^{-1}X_{1}$. If $d_{2} \geq d_{1}$ and $\Sigma_{X_{2}X_{1}}$ has a full rank, it is straightforward to see that an exact matching holds with $\beta = \Sigma_{YX_{1}}\Sigma_{X_{2}X_{1}}^{\dagger}$. Concretely, there exists $A \in \mathbb{R}^{d_{1} \times d_{2}}$ and $b \in \mathbb{R}^{d_{2}}$ such that 
\begin{equation}
    X_{2} = h(X_{1},Y) + N = AX_{1}+ bY + N,
\end{equation}
where $\E[N|X_{1},Y]=\bs{0}$.

Since the encoding function is not well-defined when $Y$ is continuous, we formulate the low-rank approximation by
\begin{align}
  \varepsilon_{s} := \min_{\hat{X}} \E\left[\left\lVert X -  \hat{X}\right\rVert^2\right]  
  &= \min_{B \in \mathbb{R}^{d_{2}\times d_{1}}: \text{rank}(B)=s} \E\left[\left\lVert\Sigma_{X_{2}X_{1}}\Sigma_{X_{1}X_{1}}^{-1}X_{1}  - B X_{1} \right\rVert^2\right]  \nonumber \\
  &= \min_{B \in \mathbb{R}^{d_{2}\times d_{1}}: \text{rank}(B)=s} \left\lVert (\Sigma_{X_{2}X_{1}}\Sigma_{X_{1}X_{1}}^{-1} - B)\Sigma_{X_1X_1}^{1/2} \right\rVert^2,  \nonumber
\end{align}
which aims to find a weighted low-rank approximation for $\Sigma_{X_{2}X_{1}}\Sigma_{X_{1}X_{1}}^{-1}$. 
 \section{Proof of Lemma~\ref{lem:bound_approx_er}}
\begin{proof}
We have
\begin{align*}
   \min_{\beta}  \mathsf{error}_{\text{apx}}(\beta) &=  \min_{\beta} \E[\lVert\E[Y|X_{1}]-\beta C^{h}(X_{1})\E[Y|X_{1}]]\rVert^2]\\
   &\leq  \min_{\beta}\E[\lVert\E[Y|X_{1}]\rVert^2]\E[\lVert \Iv_{p}-\beta(\widetilde{C}+B^{*}g^{*}(X_{1}))\rVert^2] + \varepsilon_{s}\lVert \beta\rVert^2  \\
   & \leq \min_{\beta} 2(p+\E[\lVert N \rVert^2])\left(\lVert \Iv_{p}-\beta\widetilde{C}\rVert^2 +  \lVert\beta B^{*}\rVert^2\right) + \varepsilon_{s}\lVert\beta\rVert^2 \\
   &\leq 2(p+\E[\lVert N \rVert^2])\min_{\beta} \left(\lVert \Iv_{p}-\beta\widetilde{C}\rVert^2 + \lVert\beta B^{*}\rVert^2 + \varepsilon_{s}\lVert\beta\rVert^2  \right), 
\end{align*}
where the first inequality follows from the sub-multiplicativity of the matrix norm and the triangle inequality and the last two inequalities are due to the triangle inequality, and the fact that $\E[\lVert\E[Y|X_{1}]\rVert^2]= \E[\lVert Y-N \rVert^2] \leq 2\E[\lVert Y\rVert^2] + 2\E[\lVert N\rVert^2] \leq 2(p +\E[\lVert N \rVert^2])$. When $\varepsilon_{s}=0$, note that $[\widetilde{C},B^{*}]$ has rank at most $p+s \leq d_{2}$, thus there exists at least one solution $\beta \in \mathbb{R}^{p \times d_{2}}$ for the equation $\beta[\widetilde{C},B^{*}] = [\beta\widetilde{C},\beta B^{*}] = [\Iv_{p},\bs{0}]$.   The expression of $\beta_s$ is a standard expression of ridge-type estimators.
\end{proof}

\section{Proof of Theorem~\ref{thm:ols}}\label{app:thmols}
\begin{lemma}[Concentration on the covariance matrix~\citep{lee2021predicting}]\label{lem: concen_cov} For $\bs{X} \in \mathbb{R}^{n \times d}$ with \iid rows, where each row is $\rho^2$-sub-Gaussian with covariance $\Sigma$. For any $B \in \mathbb{R}^{d \times m}$ with rank $k$ that is independent of $\bs{X}$. For any $\delta \in (0,1)$, if $n \gg \rho^4 (k+\log(\frac{1}{\delta}))$, with probability at least $1-\frac{\delta}{10}$, we have
\begin{equation}
    0.9 B^{\top} \Sigma B \preceq \frac{1}{n}   B^{\top} \bs{X}^{\top}\bs{X} B  \preceq 1.1 B^{\top} \Sigma B. \nonumber 
\end{equation}
\end{lemma}
\begin{lemma}[\citep{lee2021predicting}]\label{lem:proj}
Let $\bs{P} \in \mathbb{R}^{n \times n}$ be a projection matrix and let $\bs{Z} \in \mathbb{R}^{n \times d}$ be a matrix with \iid rows, where each row of $\bs{Z}$ is mean zero (conditioning on $\bs{P}$ being rank $k$) $\sigma^{2}$-sub-Gaussian. For any $\delta \in (0,1)$, with probability at least $1-\delta$,
 \begin{equation}
     \lVert \bs{P}\bs{Z}  \rVert \lesssim \sigma \sqrt{k(1+\log(k/\delta))}. \nonumber
 \end{equation}   

\end{lemma}

\subsection{Technical Lemmas}

\begin{lemma}\label{lemma:cov}Under Assumptions~\ref{assump:eigenvs} and~\ref{assum:lower_rank}, for any $s$ such that $1 \leq s \leq d_{2}-p$, 
\begin{itemize}
\item $\widetilde{\Sigma} \text{\, satisfies \,} \widetilde{\Sigma} \preceq \tilde{a} (1+\varepsilon_{s}) \Sigma  \text{\,\, for some \,}\tilde{a}\geq 0$\ ;
\item $\bar{\Sigma} \text{\, satisfies \,} \bar{\Sigma} \preceq \bar{a}\varepsilon_{s} \Sigma \text{\,\, for some \,} \bar{a}\geq 0$\, .
\end{itemize}
\end{lemma}

\begin{proof}
To prove  $A \preceq a B$ for symmetrical matrices $A, B \in \mathbb{R}^{d\times d}$ and $a \geq 0$, one can simply prove $\lambda_{j}(A) \leq a \lambda_{j}(B), \forall j$. This holds immediately for zero eigenvalues $\lambda_{j}(B)$'s  if $\text{rank}(A)\le\text{rank}(B)$. Therefore, by Assumption~\ref{assum:lower_rank}, we will focus on the case when $\Sigma$ has all positive eigenvalues. First, 
    \begin{align}
        \varepsilon_{s} = \frac{1}{d_{2}}\E[\lVert \widetilde{X}-X \rVert^2] \geq \frac{1}{d_{2}}\left\lvert \E[\lVert \widetilde{X}\rVert^2] - \E[\lVert X\rVert^2]  \right\rvert =\frac{1}{d_{2}} \left\lvert  \text{tr}(\widetilde{\Sigma}-\Sigma)\right\rvert \nonumber. 
    \end{align}
In the following, we will use the fact that $\lambda_{\text{min}\neq 0}(A)\Iv_{d} \preceq A \preceq \lambda_{\text{max}}(A)\Iv_{d}$ for any symmetric matrix $A \in \mathbb{R}^{d \times d}$.    
Using Assumption~\ref{assump:eigenvs},
\begin{equation}
    \widetilde{\Sigma} -\Sigma \preceq \lambda_{\text{max}}( \widetilde{\Sigma} -\Sigma) \Iv_{d_{2}} \leq \frac{c_{1}}{d_{2}} \lvert\text{tr}( \widetilde{\Sigma} -\Sigma) \rvert\Iv_{d_{2}} \leq c_{1}\varepsilon_{s} \Iv_{d_{2}}, \nonumber
\end{equation}
which implies, 
\begin{equation}
    \lambda_{i}(\widetilde{\Sigma}) \leq \lambda_{i}(\Sigma) + c_{1}\varepsilon_{s} \leq \lambda_{i}(\Sigma) + \frac{\lambda_{i}(\Sigma)}{\lambda_{\text{min}\neq 0}(\Sigma)} c_{1} \varepsilon_{s} = \left( 1+ \frac{c_{1} \varepsilon_s}{\lambda_{\text{min}\neq 0}(\Sigma)}  \right)\lambda_{i}(\Sigma) \leq \tilde{a}(1+\varepsilon_{s})\lambda_{i}(\Sigma), \nonumber
\end{equation}
for every $i$ and $\tilde{a} \geq \max(1, \frac{c_{1}}{\lambda_{\text{min}\neq 0}(\Sigma)})$. This immediately leads to $\widetilde{\Sigma} \preceq \tilde{a} (1+\varepsilon_{s}) \Sigma$. 

Finally, recall the fact that $\varepsilon_{s} = \frac{1}{d_{2}}\cdot \text{tr}(\bar{\Sigma})$. By Assumption~\ref{assump:eigenvs}, we have
\begin{equation}
    \bar{\Sigma}  \preceq \lambda_{\text{max}}(\bar{\Sigma}) \Iv_{d_{2}} \preceq \frac{c_{2}}{d_{2}}\text{tr}(\bar{\Sigma})\Iv_{d_{2}} = c_{2}\varepsilon_{s}\Iv_{d_{2}}  \preceq \frac{c_{2}}{\lambda_{\text{min}\neq 0}(\Sigma)}\varepsilon_{s} \Sigma:= \bar{a}\varepsilon_{s} \Sigma . \nonumber
\end{equation}

\end{proof}

\subsection{Proof of Theorem~\ref{thm:ols}}\label{app:thmols}
\begin{proof}
First, recall that $\beta^{*} \in \arg\min_{\beta} \mathsf{error}_{\text{apx}}(\beta)$ and the shorthand $X := \E[X_{2}|X_{1}]$. By the triangle inequality, we have
\begin{equation}
   \mathcal{R}(\hat{\beta}_{0}) = \E[\lVert\E[Y|X_{1}]-\hat{\beta}_{0}X)\rVert^2] \leq \mathsf{error}_{\text{apx}}^{*} +  \E[\lVert(\beta^{*}-\hat{\beta}_{0})X\rVert^2].\nonumber
\end{equation}
 Denote $a(X_{1}) := \E[Y|X_{1}]-\beta^{*}X$. 
 
 Recall that $X = \widetilde{X}+\bar{X}$. Then, we can write $Y = \beta^{*}\widetilde{X} +\beta^{*}\bar{X} + a(X_{1}) + N$, with $N$ satisfying $\E[N|X_{1}]=0$ by the tower property. The definition of $\hat{\beta}_{0}$ implies
\begin{equation}
 \lVert \bs{Y} - \bs{X}\hat{\beta}_{0}^{\top} \rVert^2 \leq  \lVert \bs{Y} - \bs{X}(\beta^*)^{\top} \rVert^2 = \lVert a(\bs{X}_{1})+\bs{N}\rVert^2 \nonumber. 
\end{equation}
By rearranging the terms, we get
\begin{align}
   \lVert \bs{X}(\beta^{*}-\hat{\beta}_{0})^{\top})\rVert^2 &\leq \langle a(\bs{X}_{1}), \bs{\tilde{X}}(\beta^{*}-\hat{\beta}_{0})^{\top}\rangle - \langle \bs{N},  \bs{\tilde{X}}(\beta^{*}-\hat{\beta}_{0})^{\top}\rangle  \nonumber\\
   & + \langle a(\bs{X}_{1}), \bs{\bar{X}}(\beta^{*}-\hat{\beta}_{0})^{\top}\rangle   
-    \langle \bs{N}, \bs{\bar{X}}(\beta^{*}-\hat{\beta}_{0})^{\top}\rangle
   \nonumber. 
\end{align}
We bound the first two inner products in the following, and the other two follow similarly. First,
\begin{align}
   \langle a(\bs{X}_{1}), \bs{ \tilde{X}}(\beta^{*}-\hat{\beta}_{0})^{\top}\rangle \nonumber  &=  \langle \widetilde{\Sigma}^{-1/2}\bs{ \tilde{X}}^{\top} a(\bs{X}_{1}), \widetilde{ \Sigma}^{1/2}(\beta^{*}-\hat{\beta}_{0})^{\top}\rangle \nonumber \\
   &\leq \lVert \widetilde{\Sigma}^{-1/2}\bs{ \tilde{X}}^{\top} a(\bs{X}_{1}) \rVert \lVert  \widetilde{\Sigma}^{1/2}(\beta^{*}-\hat{\beta}_{0})^{\top} \rVert \nonumber\\
   &\leq 1.1 \tilde{a} \tilde{b}\sqrt{1+\varepsilon_{s}}\sqrt{n(p+s)}\lVert  \Sigma^{1/2}(\beta^{*}-\hat{\beta}_{0})^{\top} \rVert  \nonumber\\
   & \lesssim \sqrt{1+\varepsilon_{s}} \sqrt{p+s}\lVert \bs{X}(\beta^{*}-\hat{\beta}_{0})^{\top} \rVert, \label{eq:inner1}
\end{align}
where the inequality $ \lVert \widetilde{\Sigma}^{-1/2}\bs{\tilde{X}}^{\top} a(\bs{X}_{1}) \rVert  \leq 1.1 \tilde{b}\sqrt{n(p+s)}$ is due to Assumption~\ref{assum:bound_appro} and the covariance concentration in Lemma~\ref{lem: concen_cov}, and the last inequality is simply due to the covariance concentration.
The replacement of $\widetilde{\Sigma}$ by $\Sigma$ is by
Lemma~\ref{lemma:cov}. Let $\bs{P}_{\bs{\tilde{X}}}$ denote the projection matrix defined with respect to $\bs{\tilde{X}}$, we have  
\begin{align}
    \langle \bs{N},  \bs{\tilde{X}}(\beta^{*}-\hat{\beta}_{0})^{\top}\rangle  &=  \langle \bs{P}_{\bs{\tilde{X}}}\bs{N},  \bs{\tilde{X}}(\beta^{*}-\hat{\beta}_{0})^{\top}\rangle \nonumber  \\
    &\leq \lVert\bs{P}_{\bs{\tilde{X}}}\bs{N}\rVert \lVert\bs{\tilde{X}}(\beta^{*}-\hat{\beta}_{0})^{\top}\rVert \nonumber. \\
    &\lesssim \sigma\sqrt{1+\varepsilon_{s}}\sqrt{(p+s) \left(1+\log\frac{p+s}{\delta}\right) }\lVert\bs{X}(\beta^{*}-\hat{\beta}_{0})^{\top}\rVert \label{eq:inner2},
\end{align}
where the last bound is due to Lemma~\ref{lem:proj} and the replacement of $\bs{\tilde{X}}$ by $\bs{X}$ follows from the covariance concentration and Lemma~\ref{lemma:cov}. Since we make no assumptions on the rank of $\bar{\Sigma}$, it has at most rank $d_{2}$.  Similarly, we get
\begin{align}
    \langle a(\bs{X}_{1}), \bs{\bar{X}}(\beta^{*}-\hat{\beta}_{0})^{\top}\rangle &\lesssim  \sqrt{\varepsilon_{s}} \sqrt{d_{2}}\lVert  \bs{X}(\beta^{*}-\hat{\beta}_{0})^{\top} \rVert \label{eq:inner3}\\
    \langle \bs{N}, \bs{\bar{X}}(\beta^{*}-\hat{\beta}_{0})^{\top}\rangle &\lesssim \sigma\sqrt{\varepsilon_{s}}\sqrt{d_{2}\left(1 +\log\frac{d_{2}}{\delta}\right)}\lVert\bs{X}(\beta^{*}-\hat{\beta}_{0})^{\top}\rVert. \label{eq:inner4} 
\end{align}
Combining~\eqref{eq:inner1}, ~\eqref{eq:inner2},~\eqref{eq:inner3}, and~\eqref{eq:inner4} yields
\begin{equation*}
      \lVert \bs{X}(\beta^{*}-\hat{\beta}_{0})^{\top}\rVert \lesssim \sigma \sqrt{1+\varepsilon_{s}}\sqrt{(p+s)\left(1 +\log\frac{p+s}{\delta} \right)} + \sigma\sqrt{\varepsilon_{s}}\sqrt{d_{2}\left(1 +\log\frac{d_{2}}{\delta}\right) }.
\end{equation*}
Finally, the covariance concentration implies
\begin{equation*}
    \E[\lVert(\beta^{*}-\hat{\beta}_{0})X\rVert^2] \lesssim (1+\varepsilon_{s}) \frac{(p+s) \left(1+\log\frac{p+s}{\delta}\right)}{n}\sigma^2 + \varepsilon_{s}  \frac{d_{2} \left(1+\log\frac{d_{2}}{\delta} \right)}{n}\sigma^2.
\end{equation*}
\end{proof}

\section{Proof of Corollary~\ref{coro:beta_l}}\label{app:coro9}

\begin{lemma}\label{assump:bound_d1} Under Assumptions~\ref{assump:eigenvs} and~\ref{assum:lower_rank}, for $\lambda>0$, there exists $c_{2}$ such that
   \begin{equation}
     d_{\lambda} \leq c_{2}\left(1+\frac{\varepsilon_{s}}{\lambda}\right)(p+s).  \nonumber
   \end{equation} 
\end{lemma}

\begin{remark}
It is known that $ d_{\lambda} =  \E[\lVert(\Sigma +\lambda \Iv)^{-1/2}X\rVert^2]$~\citep{hsu2012random}.
    When $\varepsilon_{s}=0$, we have $ \E[\lVert(\Sigma +\lambda \Iv)^{-1/2}X\rVert^2] \leq  \text{rank}(\Sigma) = p+s$, where the equality holds when $\lambda=0$.  
\end{remark}

\begin{proof}
Let $\{\tilde{\lambda}_{i}\}$ and $\{\lambda_{i}\}$ denote the eigenvalues of $\Sigma$ and $\widetilde{\Sigma}$, respectively. Recall that Assumptions~\ref{assump:eigenvs} and~\ref{assum:lower_rank} imply $\tilde{\lambda}_{i}- \lambda_{i} \leq c_{1}\varepsilon_{s}$ as shown in the proof of Lemma~\ref{lemma:cov}, then we have
    \begin{align}
        d_{\lambda} = \sum_{i=1}^{d}\frac{\lambda_{i}}{\lambda+\lambda_{i}} &=  \sum_{i=1}^{d}\frac{\tilde{\lambda}_{i}}{\lambda+\tilde{\lambda}_{i}} +  \lambda \sum_{i=1}^{d}\frac{\lambda_{i}-\tilde{\lambda}_{i}}{(\lambda+{\lambda}_{i})(\lambda+\tilde{\lambda}_{i})} \nonumber\\       
        &\leq \sum_{i=1}^{d}\frac{\tilde{\lambda}_{i}}{\lambda+\tilde{\lambda}_{i}} +c_{1}\varepsilon_{s} \sum_{i=1}^{d}\frac{1}{(\lambda+{\lambda}_{i})(\lambda+\tilde{\lambda}_{i})} \nonumber\\
        &= \sum_{i=1}^{d}\frac{\tilde{\lambda}_{i}}{\lambda+\tilde{\lambda}_{i}} + \frac{c_{1}\varepsilon_{s}}{\lambda \cdot \lambda_{\text{min}\neq0}(\widetilde{\Sigma})} \sum_{i=1}^{d}\frac{\lambda \cdot\lambda_{\text{min}\neq0}(\widetilde{\Sigma})}{(\lambda+\lambda_{i})(\lambda+\tilde{\lambda}_{i})} \nonumber, 
    \end{align}
where $\frac{\lambda}{\lambda+\lambda_{i}} \leq 1$ for $\forall i$ and $\sum_{i=1}^{d}\frac{\tilde{\lambda}_{i}}{\lambda+\tilde{\lambda}_{i}}  \leq p+s$ since $\text{rank}(\widetilde{\Sigma})\leq p+s$. Finally, we get
\begin{equation}
     d_{\lambda} \leq p+s +  \frac{c_{1}\varepsilon_{s}(p+s)}{\lambda \cdot\lambda_{\text{min}\neq0}(\widetilde{\Sigma})} \leq c_{2}\left(1+\frac{\varepsilon_{s}}{\lambda}\right)(p+s), \nonumber
\end{equation}
where $ c_{2} \geq \max \left\{1, \frac{c_{1}}{\lambda_{\text{min}\neq0}(\widetilde{\Sigma})}\right\}$. 
\end{proof}

\noindent{\bf Corollary~\ref{coro:beta_l}} 
\emph{Under~\citep[Condition 2 and 4]{hsu2012random}, and the assumptions in Lemma~\ref{assump:bound_d1}, the excess risk of the downstream task can be upper bounded by
    \begin{equation}
  \mathcal{R}(\hat{\beta}_{\lambda}) \leq \mathsf{error}_{\text{apx}}^{*} +\E[\lVert(\beta_{\lambda}-\beta^{*})X\rVert^2] + \mathcal{O}\left(\frac{p+s}{n}\left(1+\frac{\varepsilon_{s}}{\lambda}\right)\tilde{\sigma}^2\right), \nonumber    
\end{equation}
with high probability, where $\tilde{\sigma}^2 := \left(  \E[\lvert \E[Y|X]-\beta_{\lambda}X \rVert^2]+\E[\lVert(\beta_{\lambda}-\beta^{*})X\rVert^2]+\sigma^2 \right)$. }
\smallskip

We only outline the main steps and refer the readers to ~\citep[Theorem~16]{hsu2012random} for details. First, by the triangular inequality, we have $\mathcal{R}(\hat{\beta}_{\lambda}) \leq \mathsf{error}_{\text{apx}}^{*}  + \E[\lVert(\hat{\beta}_{\lambda}-\beta^{*})X\rVert^2]$. The bound on the second term can be obtained as discussed below. 
The last term is simply due to $d_{2,\lambda} \leq d_{1,\lambda} \leq c_{\lambda}(1+\frac{\varepsilon_{s}}{\lambda})(p+s)$ by Lemma~\ref{assump:bound_d1}, where $d_{l,\lambda} = \sum_{j=1}^{d}\left(\frac{\lambda_{j}}{\lambda_{j}+\lambda}\right)^l$ for $l \in \{1,2\}$. Observe that the ridge estimator with $Y_{j}= \mathbbm{1}_{\bar{Y}=y_j}$ as the target variable is equivalent to the $j^{th}$ row of the ridge estimator with $Y$ as the target variable. Thus we can provide a bound on $\E[\lVert(\hat{\beta}_{\lambda,j}-\beta_{j}^{*})X \rVert^2]$ for each $j \in \{1,\ldots,p\}$ according to~\citep[Theorem~16]{hsu2012random}. Summing up the inequalities gives $\E[\lVert(\hat{\beta}_{\lambda}-\beta^{*})X \rVert^2] 
     \leq   \E[\lVert(\beta_{\lambda}-\beta^{*})X \rVert^2] + 
     \E[\lVert(\hat{\beta}_{\lambda}-\beta_{\lambda})X \rVert^2]$,
 where $\E[\lVert(\hat{\beta}_{\lambda}-\beta_{\lambda})X \rVert^2]$ is upper bounded in~\citep[Theorem~16]{hsu2012random}.

\section{Synthetic Data: Details of the Data Generation}
\label{sec_app:syn}

Let $X_{1}\sim \mathcal{N}(\bs{0},I_{d_{1}})$. Note that $\E[\lVert X_{1} \rVert] = \sqrt{2}\frac{\Gamma(5.5)}{\Gamma(5)} \approx 3.08$, where $\Gamma(\cdot)$ is the Gamma function.   The label $\bar{Y}$ is determined by $\lVert X_{1}\rVert$ as follows:  $\bar{Y}=0$ when $\lVert X_{1}\rVert<2.5$;  $\bar{Y}=1$ when $2.5 \leq \lVert X_{1}\rVert<3.5$;  $\bar{Y}=2$ when $\lVert X_{1}\rVert\ge 3.5$. Then, let $X_{2} = (A + B g(X_{1}))Y+N$ where $A$ and $B$ are matrices with \iid entries from $\U[-2,2]$ and $N \sim \mathcal{N}(\bs{0},\Iv_{d_{2}})$. The $(i,j)^{th}$ element of $g(X_{1})$ is given by $(\max_{k}(X_{1,k}))\cdot\sin(\frac{2\pi i}{s}\min_{k}(X_{1,k})+\frac{2\pi j}{p})$. The sample sizes of the pretext training data and testing data are $10^4$ and $10^3$, respectively. The MLPs used for the pretext task and two SL procedures all have two fully connected hidden layers with ReLU activation. The batch size is $32$, the number of epochs is $10$, and the learning rate is $0.001$.

\section{Further Details of the Computer Vision Task}
\label{sec_app:cv_shapes}

The sample sizes of the pretext task and testing are fixed to be $20000$ and $1000$, respectively. The edge of the triangle is sampled from $\U[8,32]$, the radius of the circles is sampled from $\U[5,10]$, and the pentagon is drawn within a circle with radius samples from $\U[8,64/3]$. The sizes are chosen to ensure that the average areas of the objects are similar. The pretext task and SL both use convolution neural networks (CNNs) consisting of two convolution layers and two fully connected layers with ReLU activation. The learned representation is obtained from the second convolution layer of the CNN, which has a dimension $d_{2}=12544$. Since the rotated image $X_{2}$ has the same label as $X_{1}$, we also use the rotated images as additional labeled data for the downstream task and SL. In the pretext task, the batch size is $64$, the number of epochs for training is $15$, and the learning rate is $0.001$. For SL, the batch size the modified to be $32$ since the sample size is much smaller. For ridge regression, we choose the shrinkage parameter $\lambda$ from  $200$ numbers evenly spaced on a log scale over $[0.001, 100]$ with $5$-fold cross-validation.

In Section~\ref{sec:exp_cv}, we explain that the unsatisfactory performance of SSL for Triangle vs. Pentagon is potentially due to the analogous characteristics (i.e., the edges and vertices) that determine the orientation of the object. From a theoretical perspective, the full rank condition on $C^{h}(x)$ (that is a necessary condition for the exact matching) is violated since the columns of $C^{h}(x)$ are approximately linearly dependent. This does not happen for Triangle vs. Tangent Circles since the orientation of the circles is determined by curved edges. To further support our observation, we use Grad-CAM~\citep{selvaraju2017grad} to visualize the contributing features that the pretext model used for rotation prediction. To highlight the major components, we only show pixels of the heatmap with top-20\% intensity in Fig.~\ref{fig:heat_shapes}.

\begin{center}\includegraphics[width=.8 \linewidth]{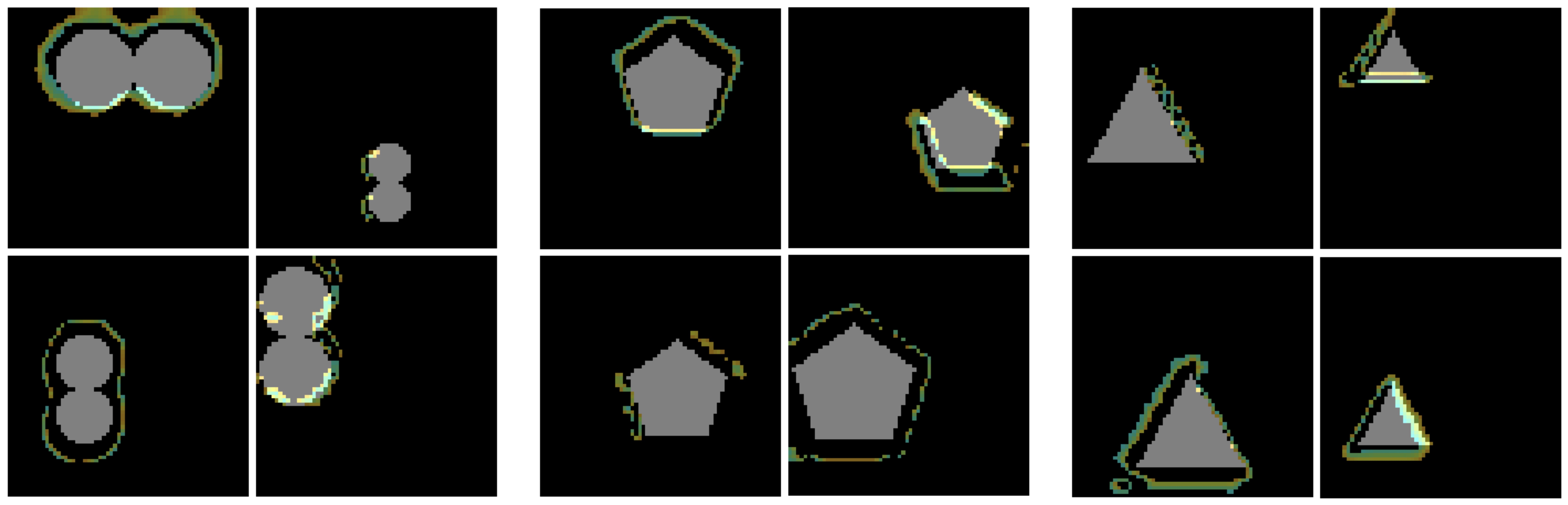}
  \captionof{figure}{\small Visualization of the contributing features for rotation prediction in the computer vision task.  \label{fig:heat_shapes}
  } 
  \end{center}

Similarly to the stylized MNIST dataset, we add dot and dash patterns to the image background. The space in patterns is sampled from $\U[8,32]$. 

\begin{center}\includegraphics[width=.6 \linewidth]{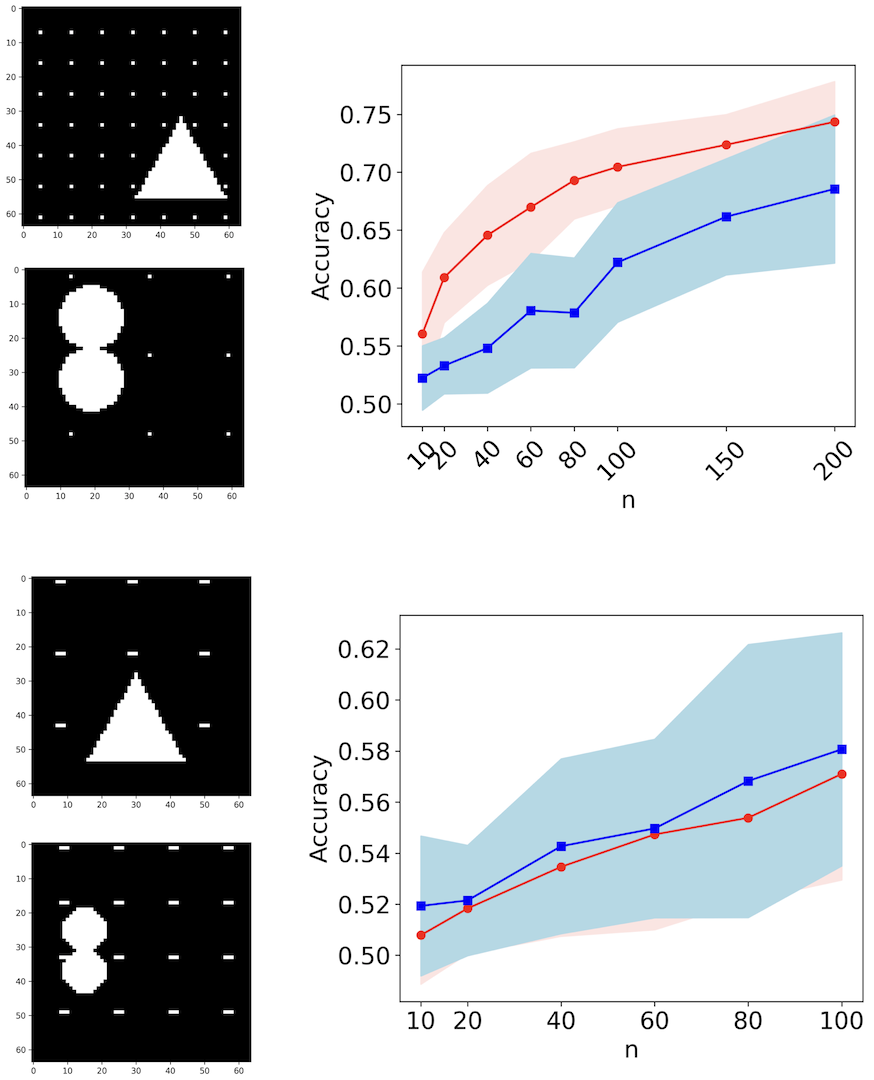}
  \captionof{figure}{\small Geometric shape images with dot vs. dash background. SSL (red) and SL (blue) \label{fig:geo_pat_resutls}
  } 
  \end{center}
  
\newpage
\section{Experiments on the MNIST Dataset}
\label{sec_app:mnist}
The sample sizes for the pretext task and testing is $45000$ 
and $10000$, respectively. We compare the performance of SSL and SL under different labeled sample sizes $\{50,100,200,400\}$.  The space $d$ in the sparse and dense patterns is sampled from $\U[3,8]$ and $\U[8,15]$, respectively. We randomly shift each pattern by $\U[0,d]$ to avoid the position of the pattern being correlated with the image orientation.  We use the same CNN configuration as the geometric shape task.

\begin{center}\includegraphics[width=.75 \linewidth]{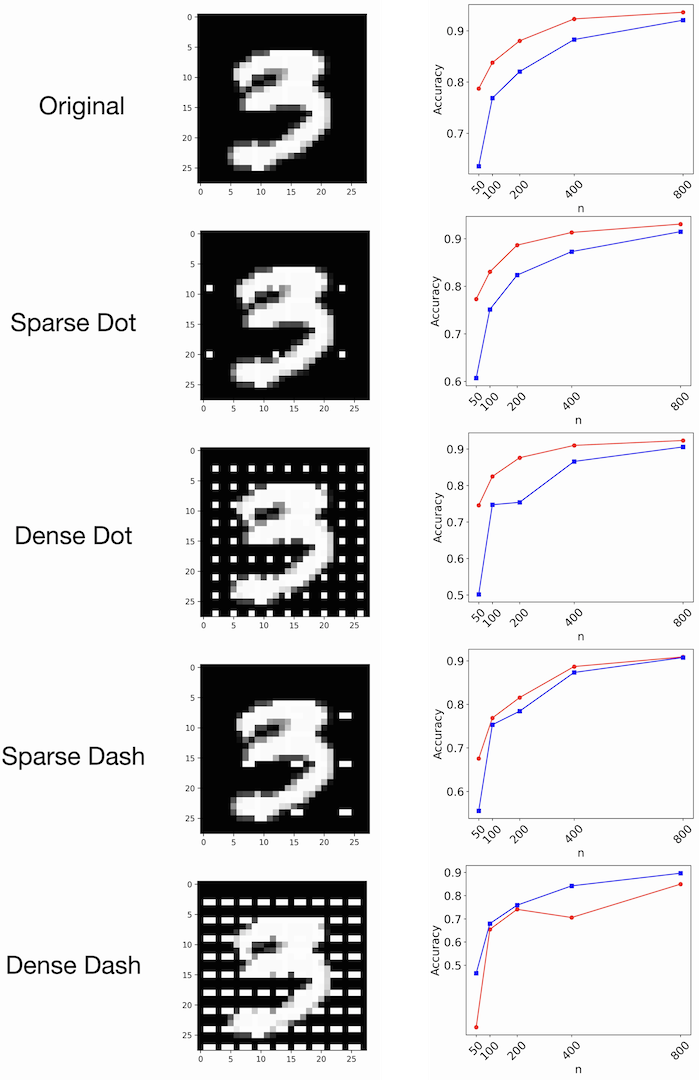}
  \captionof{figure}{\small The original MNIST dataset and four stylized versions of MNIST. SSL (red) and SL (blue).  \label{fig:mnist_resutls}
  } 
  \end{center}

\begin{center}\includegraphics[width=.8 \linewidth]{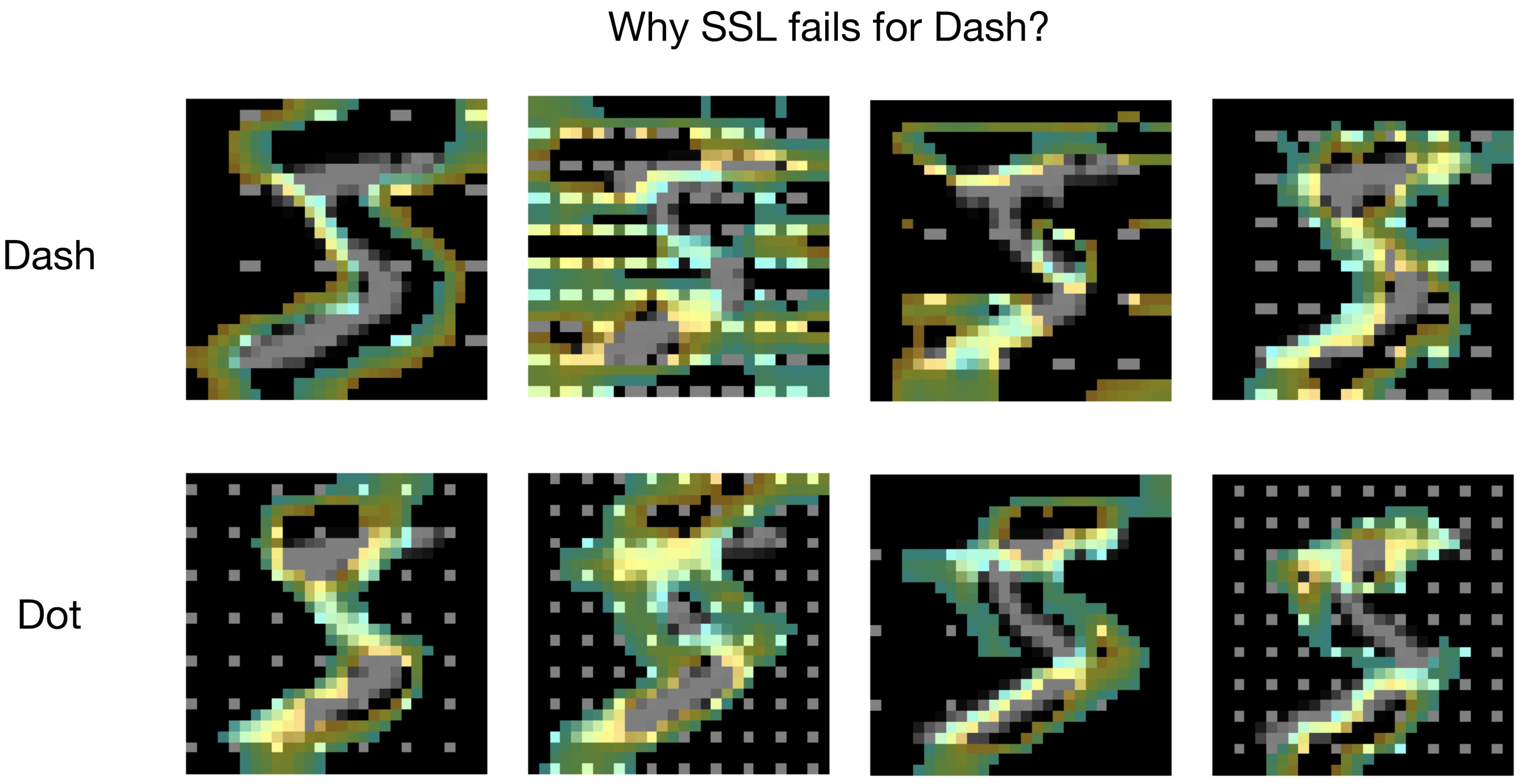}
  \captionof{figure}{\small A visualization of the contributing features for MNIST dataset with dash vs. dot background. 
  \label{fig:mnist_heat}
  } 
  \end{center}

\end{document}